\documentclass[journal]{./IEEEtran}

\usepackage{times}
\usepackage{soul}
\usepackage{url}
\usepackage[hidelinks]{hyperref}
\usepackage[utf8]{inputenc}
\usepackage[small]{caption}
\usepackage{graphicx}
\usepackage{amsmath}
\usepackage{amsthm}
\usepackage{booktabs}
\usepackage{multirow}
\usepackage{subfigure}
\usepackage{setspace}
\usepackage{color}
\usepackage{subfloat}
\usepackage{algorithmic}
\usepackage{makecell}
\usepackage[linesnumbered,ruled]{algorithm2e}
\usepackage{amsthm}
\newcommand*{\bs}[1]{\boldsymbol{#1}}
\usepackage{mathtools}
\usepackage{cite}
\usepackage[normalem]{ulem}
\useunder{\uline}{\ul}{}

\newtheorem{thm}{Theorem}[section]

\newtheorem{prop}[thm]{Proposition}

\usepackage{latexsym}
\usepackage{amssymb}

\begin{document}
%
\title{Robust Group Anomaly Detection for Quasi-Periodic Network Time Series}
%
%
%

\author{Kai Yang, Shaoyu Dou, Pan Luo, Xin Wang, H. Vincent Poor
\thanks{Kai Yang, Shaoyu Dou, Pan Luo are with the Department of Computer Science and Technology, Tongji University, Shanghai, China. Xin Wang is with School of Information Science and Technology, Fudan University, Shanghai, China. H. Vincent Poor is with the Department of Electrical Engineering, Princeton Universty, USA. 
(Corresponding author: Kai Yang).}
}

%
%

\markboth{Journal of \LaTeX\ Class Files,~Vol.~14, No.~8, August~2015}%
{Shell \MakeLowercase{\textit{et al.}}: Bare Demo of IEEEtran.cls for IEEE Journals}
%



\maketitle

\begin{abstract}
Many real-world multivariate time series are collected from a network of physical objects embedded with software, electronics, and sensors. The quasi-periodic signals generated by these objects often follow a similar repetitive and periodic pattern, but have variations in the period, and come in different lengths caused by timing (synchronization) errors. Given a multitude of such quasi-periodic time series, can we build machine learning models to identify those time series that behave differently from the majority of the observations? In addition, can the models help human experts to understand how the decision was made? We propose a sequence to Gaussian Mixture Model (seq2GMM) framework. The overarching goal of this framework is to identify unusual and interesting time series within a network time series database. We further develop a surrogate-based optimization algorithm that can efficiently train the seq2GMM model. Seq2GMM exhibits strong empirical performance on a plurality of public benchmark datasets, outperforming state-of-the-art anomaly detection techniques by a significant margin. We also theoretically analyze the convergence property of the proposed training algorithm and provide numerical results to substantiate our theoretical claims.
\end{abstract}

\begin{IEEEkeywords}
\textcolor{black}{Group anomaly detection, Timing errors, Gaussian Mixture Model}
\end{IEEEkeywords}

%
\IEEEpeerreviewmaketitle

\begin{spacing}{0.98}

\section{Introduction}
\label{section:1}

\IEEEPARstart{T}{ime}
series anomaly detection is an important problem for both fundamental signal processing and machine learning research and a variety of engineering applications, including health informatics \cite{ECG-nature-medicine}, sensor data management \cite{Basu2007}, and service performance monitoring \cite{Xu2018,yang2016deep}. Two basic approaches to anomaly detection are those based on statistical models and those based on machine learning methods applied directly to data.
In this paper, we are interested in the latter type of anomaly detection algorithms. Such learning-based anomaly detection methods detection can be broadly categorized into three groups, i.e., supervised, unsupervised, and semi-supervised learning techniques. Anomalies by their nature happen sporadically and remain difficult to detect. The lack of suitable training data makes the supervised learning approach \cite{Steinwart:2005:CFA:1046920.1058109} inapplicable in many cases. Instead, unsupervised machine learning approaches, such as K Nearest Neighbor (KNN), one-class support vector machine (OCSVM) \cite{Amer2013}, Local Outlier Factor (LOF) \cite{breunig2000lof}, Deep Autoencoder Gaussian Mixture Model (DAGMM) \cite{zong2018deep}, are capable of discovering unusual patterns within data without any training labels and thus are particularly suitable for a wide range of applications in which the anomaly labels are difficult to obtain.

We assume the quasi-periodic time series under investigation is obtained via sampling an underlying continuous signal. It can be loosely defined as a time series of recurring patterns with different lengths, a.k.a. rhythms or motifs \cite{motifDiscovery,BeatGAN}. A notable example is the electrocardiogram (ECG) signal. Irregular ECG patterns can be used to diagnose cardiac arrhythmia, which causes a tremendous number of deaths around the world every year.
Therefore, detecting unusual quasi-periodic time series with timing errors in a time series database is still an underappreciated topic. To the best of our knowledge, no existing works have addressed this particular problem so far.

As a common phenomenon in the sampling process, \emph{timing (synchronization) error} brings great challenges against analyzing the quasi-periodic time series \cite{jin2000digital}. For example, for IoT devices equipped with low-quality hardware, the clock may drift considerably between synchronization points which leads to timing errors. 
We consider the following two types of timing errors. Type-1 timing error refers to the instantaneous time shift of the signal caused by clock jitters \cite{wolaver1991phase}.
Type-2 timing errors manifest itself as signal deletion due to unreliable sampling elements \cite{yang2020decoding,mitzenmacher2009survey}. Fig. \ref{fig:timing_error} gives examples of these two types of timing errors.
Such timing errors are common in a multitude of applications as diverse as IoT data management and health informatics.

\begin{figure}[!tbp]
  \centering
    \centering
    \includegraphics[width=3in]{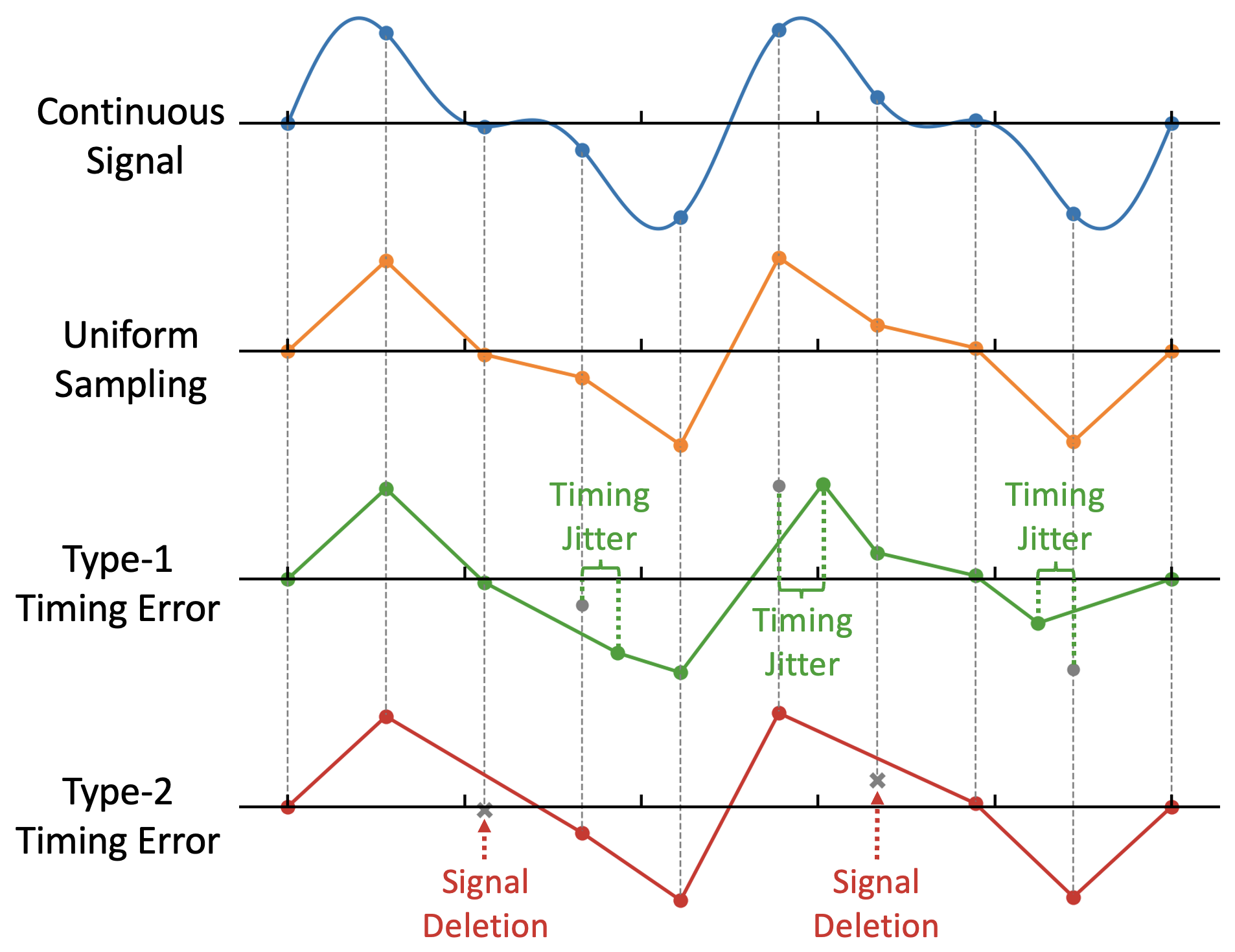}
  \centering
  \caption{Exemplary time series samples with Type-1 and Type-2 timing errors.}
  \label{fig:timing_error}
\end{figure}

The shapelet is the most discriminative small subsection of the shape between two classes of shapes. Thus the time series shapelet can be loosely defined as a subsequence which is the defining difference between two classes of time series. Time series shapelet was first proposed in \cite{ye2009time} for providing visual and interpretable outcomes in time series classification tasks. \cite{beggel2019time} provides the first application of shapelet in unsupervised time series anomaly detection by learning normal shapelets from the training set, which maximally represent the normal class. In contrast, the \emph{anomaly shapelets} can be viewed as those subsequences that cause samples to be deemed as anomalies. They can be considered as the root cause of anomalies, therefore localizing anomaly shapelets can help humans understand the detected anomalies.
As opposed to the traditional approaches, We learn shapelet in the latent space through a recurrent encoder-decoder architecture, therefore we call it \emph{recurrent shapelet}.

The group behavior anomaly detection problem that we aim to address is described as follows. Also, for brevity, we call quasi-periodic time series rhythmic time series hereafter.

\textbf{Problem 1 (Robust Group Anomaly Detection for Rhythmic Time Series)}:
Given a collection of $N$ univariate rhythmic time series $\{\bs{s}_1, \bs{s}_2, ... , \bs{s}_N\}$ with different lengths, the primary goal of seq2GMM is to 
1) identify anomalous time series that appear to be different from the normal pattern, 
2) pinpoint the anomaly subsequences, i.e., anomaly shapelets in $\bs{s}_i$ that can help users to investigate and understand the detected anomalous time series.

\subsection{Running Examples}

We briefly discuss a collection of application domains of the proposed framework and also provide a few illustrative
examples.

\begin{enumerate}

\item \textbf{IoT Data Management}: Sensor data captured from IoT devices plays an essential role in many real-life IoT applications such as cloud server performance monitoring \cite{Xu2018}, inventory management systems, and human body posture analysis. For example, the CMU motion capture database consists of motion capture records including walking, running, jumping. Our proposed framework can help identify anomalous behaviors of human posture or cloud service performance.

\item \textbf{Health Informatics}: Health informatics uses information technologies to process and analyze health and medical information of patients to improve the healthcare outcomes. A notable example in health informatics is real-time ECG signal analysis which aims to continuously monitor cardiovascular patients on a daily basis, identify abnormal behaviors, and warn the patients in case of cardiac arrhythmia.

\item \textbf{Shape Analysis}: Shape analysis refers to the automatic process of analyzing geometric shapes. For example, shape analysis allows detection of objects of similar shapes in a database. It finds a variety of applications in security applications, e.g, face recognition, and medical imaging, e.g, shape changes due to illness.

\end{enumerate}

Seq2GMM can be employed for face recognition, as shown in the following example. The facial contours can be mapped into a univariate time series through preprocessing, as shown in Fig. \ref{fig:facesucr}. In this example, various segments delineate the contours of different parts, including head, face, neck, etc. Seq2GMM aims to identify unseen faces in the database, which can be viewed as an anomaly. It also seeks to localize the anomaly shapelets to identify which part of the contour appears to be abnormal.

\begin{figure}[!tbp]
  \centering
    \centering
    \includegraphics[width=2.3in]{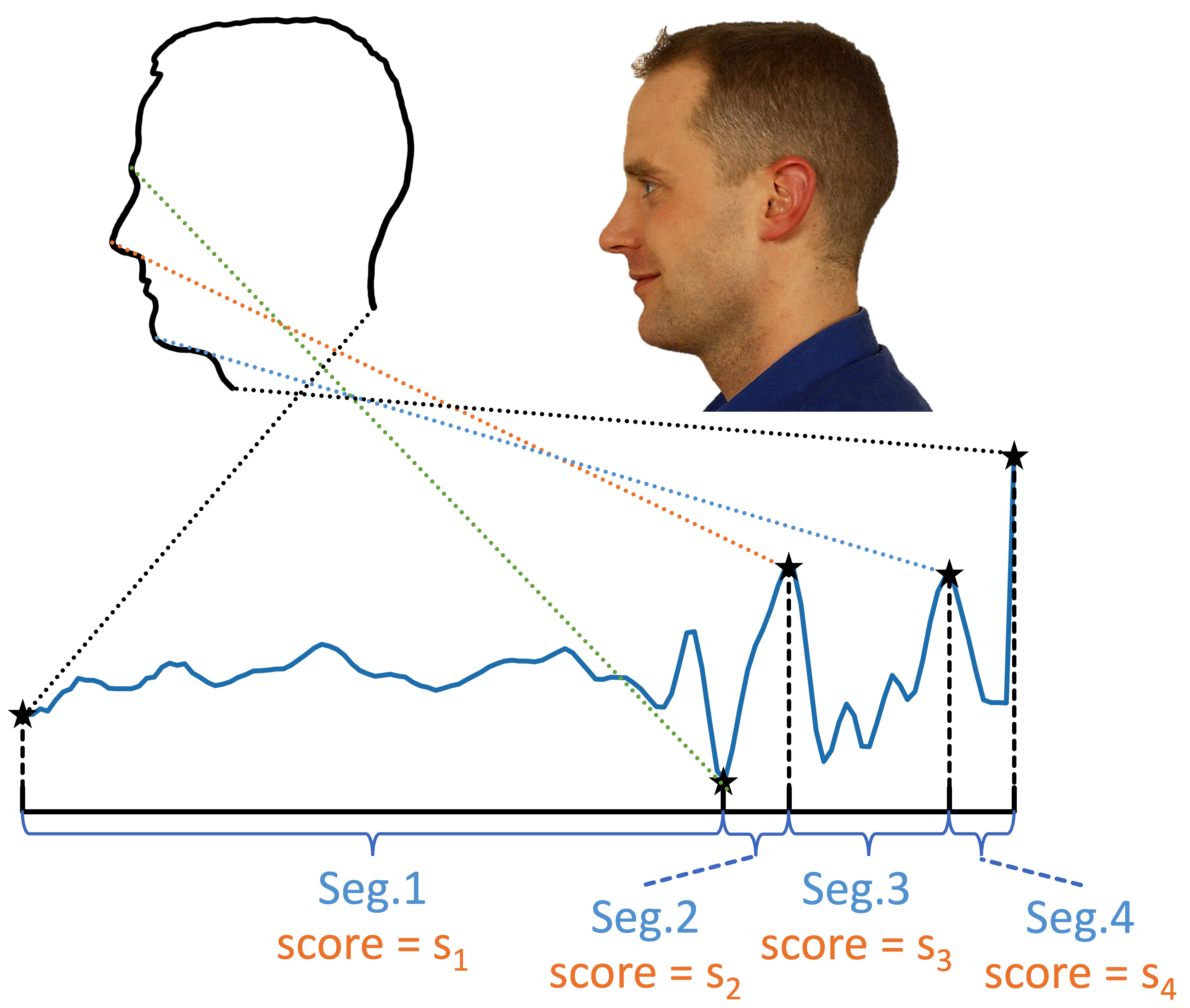}
  \centering  \caption{Example of seq2GMM used for shape analysis.}
  \label{fig:facesucr}
\end{figure}

Next, we show how an example in which seq2GMM can be employed for ECG signal analysis. Analyzing ECG signal and detecting anomalies is critical for heart disease diagnosis. There are a total of six types of ECG waves, including P wave, Q wave, R wave, S wave, T wave and U wave. Therefore, it is of great importance to not only detect unusual ECG signal, but also pinpoint the segment that appears to be abnormal. 
Fig. \ref{fig:ecg_ecample} shows a normal beat and an abnormal beat of the ECG V2 signal, where the abnormal beat exhibits significant anomalous pattern in the RS segment, which indicates a potential heart disease. Seq2GMM analyzes ECG signals by learning normal signals and identifying unusual signals that appear to be abnormal. It identifies anomalies by jointly mining the dynamics of each wave. In addition, seq2GMM outputs an anomaly score for each wave, which characterizes the level of abnormality a particular wave is associated with.

\begin{figure}[!tbp]
  \centering
    \centering
    \includegraphics[width=1.9in]{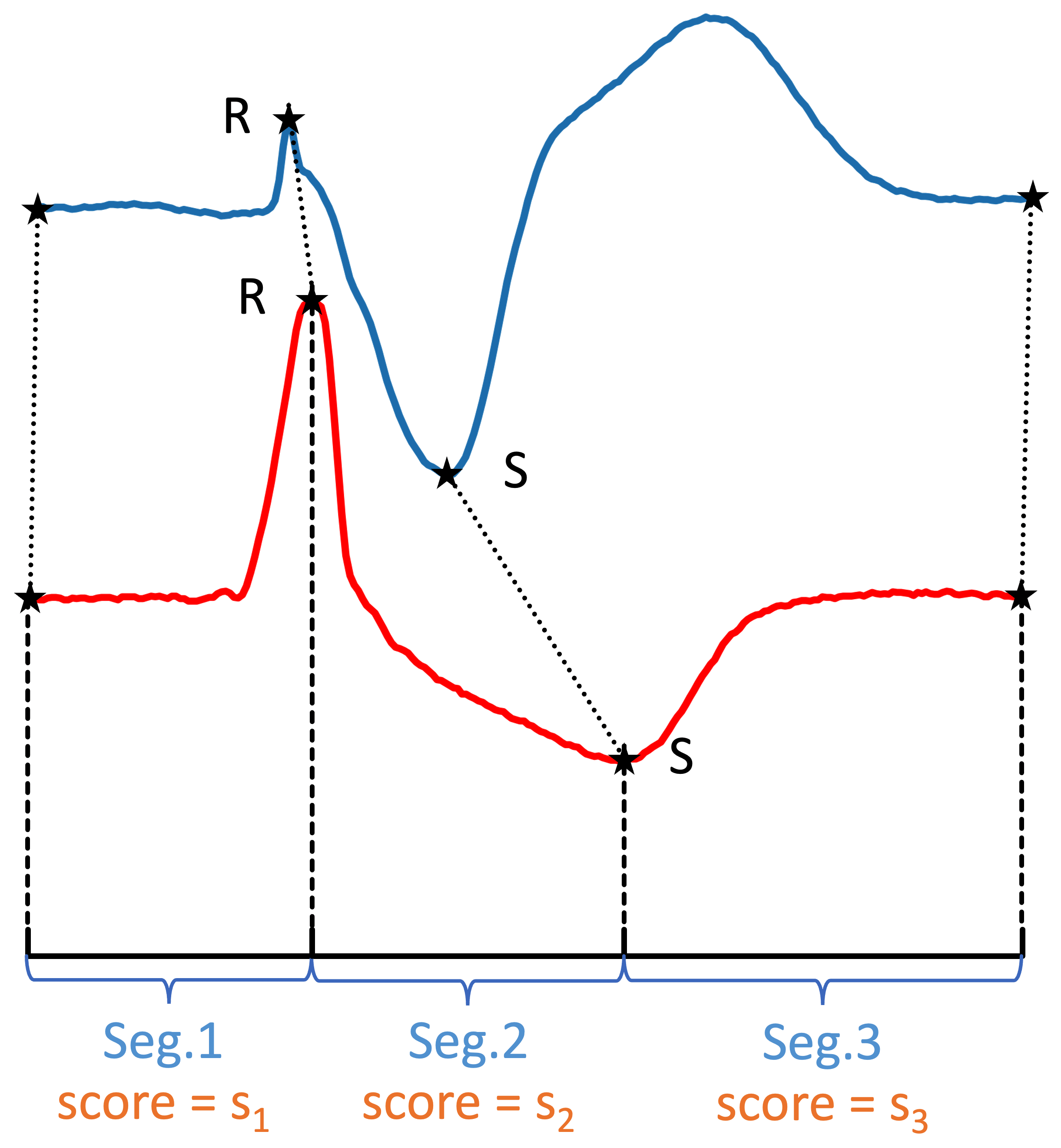}
  \centering  \caption{An example of a normal beat (blue line on the top) and an abnormal beat (red line on the bottom), according to the AAMI EC57 standard.}
  \label{fig:ecg_ecample}
\end{figure}

Notice that in aforementioned examples, it is difficult to guarantee that all collected sequences are of equal length. It remains a significant challenge to align them perfectly due to timing errors. These constraints motivate us to design robust anomaly detection algorithms that remain effective in the presence of timing errors.

\subsection{Contributions}

We design a deep learning framework a.k.a. seq2GMM that is capable of conducting clustering and unsupervised anomaly detection for rhythmic time series coming with timing errors. More concretely, we have made the following contributions.

\begin{enumerate}

\item \textbf{Learning without anomaly training samples}: Seq2GMM requires no training labels and shows strong empirical performance. It outperforms state-of-the-art algorithms by a large margin over a plurality of public datasets from real world scenarios.

\item \textbf{Rhythmic time series with timing errors}: While the recurring patterns of different rhythmic time series are in general similar to each other, it is difficult to align them perfectly due to the timing errors. Seq2GMM can effectively learn the similarities among a multitude of rhythmic time series with timing errors. In addition, we also provide data augmentation methods to further improve the robustness of the model.

\item \textbf{Visualization and localization}: As shown in Fig. \ref{fig:split2}, the anomalies are detected in the latent space which allows better visualization.
Furthermore, seq2GMM provides \emph{local interpretability} via identifying the recurrent anomaly shapelets within the time series. 
The recurrent anomaly shapetlet provides rich semantic information for users to understand the reason that such a time series is deemed as being anomalous.
\end{enumerate}

Moreover, we have developed a surrogate-based model training algorithm with convergence guarantee. The training algorithm can achieve a stable learning model with comparable performance even in the presence of up to $10\%$ anomalies in the training data.

The remainder of this paper is organized as follows. We begin this paper in Section \ref{section:2} via providing a brief review of the related work. In Section \ref{section:3}, we present the anomaly model in latent space. Next, in Section \ref{section:4}, sequence to GMM learning via deep attentive neural network, data augmentation and anomaly shapelet localization methods are developed. We also assess the performance of the proposed model via extensive experiments on real-world datasets in Section \ref{section:5}. We finally conclude this paper in Section \ref{section:6}.

\section{Related Work}\label{section:2}

Considerable progress has been made for supervised learning of time series anomaly detection. It can achieve high accuracy if sufficient labeled training data is provided \cite{ECG-nature-medicine,yang2018active}. However, anomaly training data is usually difficult to obtain in practice. And the resultant model obtained by a supervised approach might not be able to identify anomalous patterns that never appeared before. Therefore, an unsupervised learning approach is more desirable in practical applications when there is insufficient training data.

Unsupervised anomaly detection has received significant research interest over the past decade and there exists a large body of work \cite{chandola2009anomaly}. There is recently a growing interest in applying deep learning techniques to unsupervised anomaly detection. There exist two major approaches, i.e., prediction-based and classification-based methods for unsupervised anomaly detection. The prediction-based method analyzes time series data using deep learning model to make predictions. Deviations between the predicted and true values are deemed as anomaly scores \cite{hundman2018detecting,dou2019pc}. The latter determines whether a given time series is anomalous based on its latent space representation. Such an approach has received increasing attention because of its ability to identify anomalies from the latent space \cite{zong2018deep,ruff2018deep,erfani2016high,shen2020timeseries}. 

Learning representations from unlabeled time series data is a key yet very challenging task. The existing time series representation learning methods can be grouped into two categories according to whether they employ deep learning methods or not. Non-deep learning methods, such as spectral decomposition \cite{golub2013matrix,percival2000wavelet}, frequency domain analysis \cite{faloutsos1994fast,chan1999efficient}, and symbolic methods \cite{lin2003symbolic,schafer2012sfa}, are in general computationally more efficient than deep representation learning methods. However, non-deep learning approaches usually yield performance that is inferior to that of the deep representation learning methods.
While a variety of deep learning models \cite{paparrizos2019grail,yuan2019wave2vec,fortuin2018som} have been proposed specially for time series representation learning, deep autoencoders have been deemed as the dominant representation learning framework, because of its robustness against anomaly training samples as well as its low computational cost.
Along the line of research regarding the encoder/decoder structure, existing works include RNN-based \cite{ma2019learning,yao2017trajectory}, CNN-based \cite{franceschi2019unsupervised,BeatGAN}, and self-attention-based \cite{zerveas2021transformer,zhou2021informer}.
Such unsupervised anomaly detection methods that combine time series representations learned by deep autoencoders are gradually becoming mainstream.
\cite{zhang2017joint,zong2018deep} employ a deep autoencoder to obtain a low-dimensional representation in the latent space, and then utilize a Gaussian mixture model (GMM) to estimate the probability distributions and subsequently detect the outliers. Likewise, a generative adversarial network (GAN) has been combined with the deep encoder-decoder structure for detecting anomalies in rhythmic time series \cite{BeatGAN}. 

Besides detecting time series anomalies, identifying and localizing the most anomalous segments, i.e., anomaly shapelets within the time series are equally important. However, most classification-based methods only generate an anomaly score for a given time series and therefore cannot localize the anomaly shapelet, if it exists. 
A possible approach is to divide the entire time series into multiple segments and subsequently detect the anomalous segment. Specifically, the segmentation algorithm, such as IEP \cite{fink2007important}, PIP \cite{fu2006time} and PLSA \cite{fuchs2010online}, etc., divide the time series into multiple subsequences according to their temporal dynamics. Then machine learning algorithms are carried out on these subsequences for anomaly detection.
\cite{sivaraks2015robust,thuy2021efficient,pang2018intelligent} utilize KNN based on DTW or other distance metrics to identify and localize anomaly shapelets on each set of subsequences. \cite{yang2013trasmil} utilizes hierarchical dirichlet process HMM (HDP-HMM) to learn representations of subsequences, which are used to train supervised classification models to support anomaly shapelets localization.
In addition, temporal segmentation is also employed as data preprocessing or dimensionality reduction methods to support data mining on lengthy quasi-periodic time series \cite{liu2020anomaly}.
Different from existing methods, we propose in this paper a recurrent anomaly shapelet localization approach by leveraging the powerful representation capabilities of deep learning model.

While various efforts have been undertaken to discover temporal anomalies in time series \cite{1996:temporal:anomaly,hyndman2015large,chatfield2016analysis}, to the best of our knowledge, leveraging unsupervised deep learning techniques for the detection of anomalies in quasi-periodic time series while considering timing errors has not been systematically investigated before. We have summarized in Table \ref{tab:dataset} existing works in the literature. It is evident that only seq2GMM meets all the desired properties.

\begin{table}[!t]
  \centering
  \caption{Characteristics of existing related work}\label{tab:dataset}
  \scalebox{0.9}{
  \begin{tabular}{cccccc}
    \toprule
     & Localization & Timing Errors & Nonlinear \\
    \midrule
      DL-OCSVM &   &   & \checkmark \\
      AE &  &  & \checkmark \\
      DAGMM &     &   & \checkmark \\
      BeatGAN &  \checkmark  &  & \checkmark \\
      seq2GMM &  \checkmark  &   \checkmark & \checkmark\\
    \bottomrule
  \end{tabular}
    }
\end{table}

\begin{figure*}[t]
  \centering
  \includegraphics[height=2.0in]{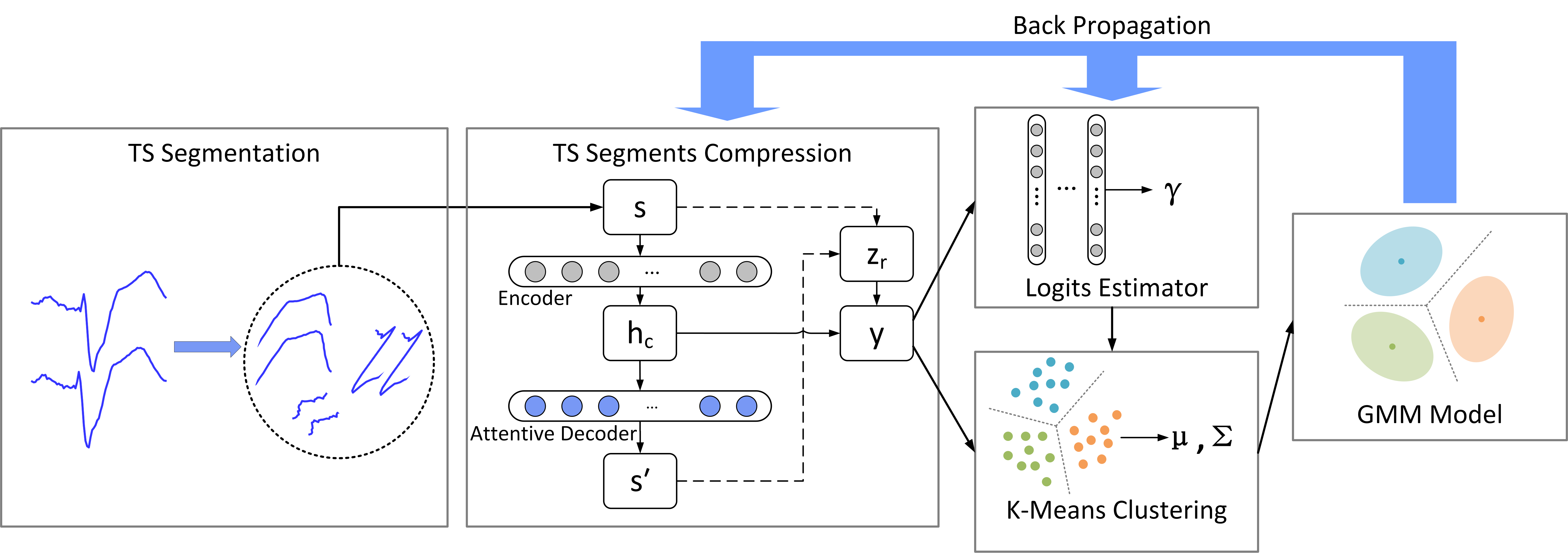}
  \caption{Graphical illustration of seq2GMM. seq2GMM is composed of three blocks,
  namely the time series segmentation, the temporal compression network,
  the estimation network with GMM.}
\label{fig:structure}
\end{figure*}

\section{Group Anomaly Model in Latent Space}\label{section:3}
In this paper, vectors are represented in lowercase boldface, scalars are in italics, and uppercase boldface is used for matrices. We assume there are a collection of $N$ sequences $\{\bs{s}_1, \bs{s}_2, ... , \bs{s}_N\}$, called a bundle \cite{Matsubara2014a}. Different from \cite{Matsubara2014a}, there may exist unknown timing shifts between these time series and they may be of different lengths. For a time series $\bs{s}_i$, it is divided into multiple temporal segments $\bs{s}_i=\{\bs{s}_{i1}, \bs{s}_{i2}, ... , \bs{s}_{iM}\}$. Each segment is compressed into a latent space through an encoder-decoder structure with attention mechanism, as discussed in the next section. We assume the latent space representation of a normal time segment $s$ is drawn from a null hypothesis, i.e., $\texttt{H}_0 \sim h_o(\bs{s})$. We assume $h_0(\bs{s})$ follows a Gaussian mixture model. GMM is employed because it is a powerful and flexible model that has been proved to be able to approximate any continuous distribution with a finite number of mixtures arbitrarily well \cite{GMM-book}. Thus the key is to efficiently learn a compact latent space representation from the training data. Please note that the proposed approach essentially consists of three steps: 
1) temporal segmentation through a piecewise linear regression model;
2) a temporal sequence compression network to extract the temporal dynamics from the time series;
3) a Gaussian Mixture Model in the latent space. 
The parameters of steps 2 and 3 need to be jointly optimized. It is also imperative to have the temporal segmentation module since it allows the machine learning model to effectively extract the temporal dynamics and emphasize on the most anomalous segment of each time series.

\section{Sequence to GMM Learning via Deep Attentive Neural Network}\label{section:4}

Given a time series $\bs{s}$ composed of multiple temporal segments $\bs{s}_i=\{\bs{s}_{i1}, \bs{s}_{i2}, ... , \bs{s}_{iM}\}$, seq2GMM aims to group all obtained temporal segments into different clusters according to the low-dimensional representation $\bs{y}$ of each segment $\bs{y}_i=\{\bs{y}_{i1}, \bs{y}_{i2}, ... ,\bs{y}_{iM}\}$ and associate each segment with an anomaly score $\{e_{i1}, e_{i2}, ... , e_{iM}\}$.

\subsection{Model Overview}
As shown in Fig. \ref{fig:structure}, seq2GMM consists of three modules, namely the temporal segmentation, temporal compression network, and GMM estimation module.

\subsection{Temporal Segmentation}
\label{sec:segment}

A time series can often be represented as a sequence of finite-length segments.
Here we employ a piecewise linear regression method to pre-process the time series under investigation and partition it into a sequence of segments. The primary goal of the proposed temporal segmentation method is elucidated in Fig. \ref{fig:split1}. In particular, this pre-processing step allows us to learn the temporal dynamics of \emph{each segment} instead of the entire time series. In doing so, the machine learning model can emphasize on and pinpoint the most anomalous segment of a time series.

For each time series, we split it into non-overlapping temporal segments \(\{\bs{s}_{i1}, \bs{s}_{i2}, ... , \bs{s}_{iM}\}\). Note that for each time series in the same data set, they share the same $M$. We use a piecewise linear regression model to approximate the time series. The loss function for segmentation is composed of two components, i.e., the least square error between the temporal sequence and its linear representation as well as the ratio between the within-class scatter matrix and between-class scatter matrix. Without loss of generality, we omit the subscript of $\bs{s}_i$ for brevity hereafter. Recall that $\bs{s}$ is a time series of arbitrary length and let $\bs{l}$ denote the corresponding piecewise linear regression. The loss function between the time series and its piecewise linear regression is given by $L_S = E_R(\bs{s}, \bs{l})$. Let $b_i$ and $\beta_i$ denote respectively the starting point and slope of the \(i^{th}\) segment, and we assume there are a total of $M$ segments. The following piecewise linear regression model is adopted.

\begin{equation*}
  l(i) =
    \begin{dcases}
      \begin{array}{ll}
        \beta_1+\beta_2(i-b_1),
        & b_1 \leq i \leq b_2, \\
        \beta_1+\beta_2(i-b_1)+\beta_3(i-b_2),
        & b_2 \leq i \leq b_3, \\
        ...  \quad   ... \\
        \beta_1+...+\beta_{M+1}(i-b_{M}),
        & b_M \leq i \leq b_{M+1}.\\
      \end{array}
    \end{dcases}.
\end{equation*}
where $l(i)$ represents the regression result for the $i^{th}$ data point. The above equations can be compactly represented by a linear system
\(\bs{A}\bs{\beta} = \bs{l}\) in which \(\bs{A}\) is the regression matrix
\begin{equation}
  A_{ij} =
  \begin{dcases}
    i-b_j, & \text{if}~~i>b_j,\\
    0,  & \text{otherwise}.\\
  \end{dcases}
\end{equation}
Consequently, let $\left(\cdot \right)^T$ denote the matrix transpose, the sum of square error is given by $ E_R(\bs{s}, \bs{l}) = \bs{e}^T\bs{e}$, where \(\bs{e} = \bs{A}\bs{\beta} - \bs{s}\) and the optimal solution that can minimize the sum of square error is given by \(\bs{\beta}^* = (\bs{A}^T\bs{A})^{-1}\bs{A}^T\bs{s}\). Notice that $b_1$ and $b_{M+1}$ are the starting and end points of the time series respectively.
The breakpoint optimization problem can be written as
\begin{equation} \bs{b}^* =\arg\min_{\bs{b}} E_R\left(\bs{s}, \bs{l}(\bs{b})\right) = \arg\min_{\bs{b}}\bs{e}(\bs{b})^T\bs{e}(\bs{b}),\end{equation}
where $\bs{b} = [b_2, \cdots, b_M]$.
We carry out the following greedy algorithm to solve the above problem and choose a proper $M$ before training the neural network. It progressively adds breakpoints to divide the time series in the training set. For example, in the first iteration, we partition each time series in the training set into two segments. Then we perform K-means clustering on all the temporal segments, where $K$ equals to 2. We subsequently calculate the Calinski-Harabasz index to evaluate the performance of clustering. A higher index implies a better clustering output. We then add another breakpoint to each time series to divide it into three segments. We calculate a Calinski-Harabasz index for each iteration. This algorithm stops and outputs the optimized hyperparameter $M^*$ when this index stars to decrease.

For the running example 2 given in Section \ref{section:1}, the temporal segmentation splits each heartbeat signal into $M$ non-overlapping subsequences. And the breakpoints are semantically meaningful. Fig. \ref{fig:ecg_ecample} shows the case where a normal beat and an abnormal beat are split into three segments, i.e. $M=3$. In this case, the breakpoints are the start and end points of the RS segment. Please note that under the above segmentation approach, the $i$-th subsegment of all beats follows a similar pattern and is quite different from the other subsegments. Segmentation also plays two other key roles in seq2GMM. First, it allows us to localize the anomalous shapelet within a time series which can help users investigate and understand the abnormal patterns within the time series data set. 
In addition, partitioning a single time series into multiple segments allows to mine the temporal dynamics of these segments jointly in a multi-task learning manner, which can effectively strengthen the adversarial robustness of the model \cite{mao2020multitask}.

\begin{figure}[!tbp]
  \centering
    \centering
    \includegraphics[height=2in,width=3in]{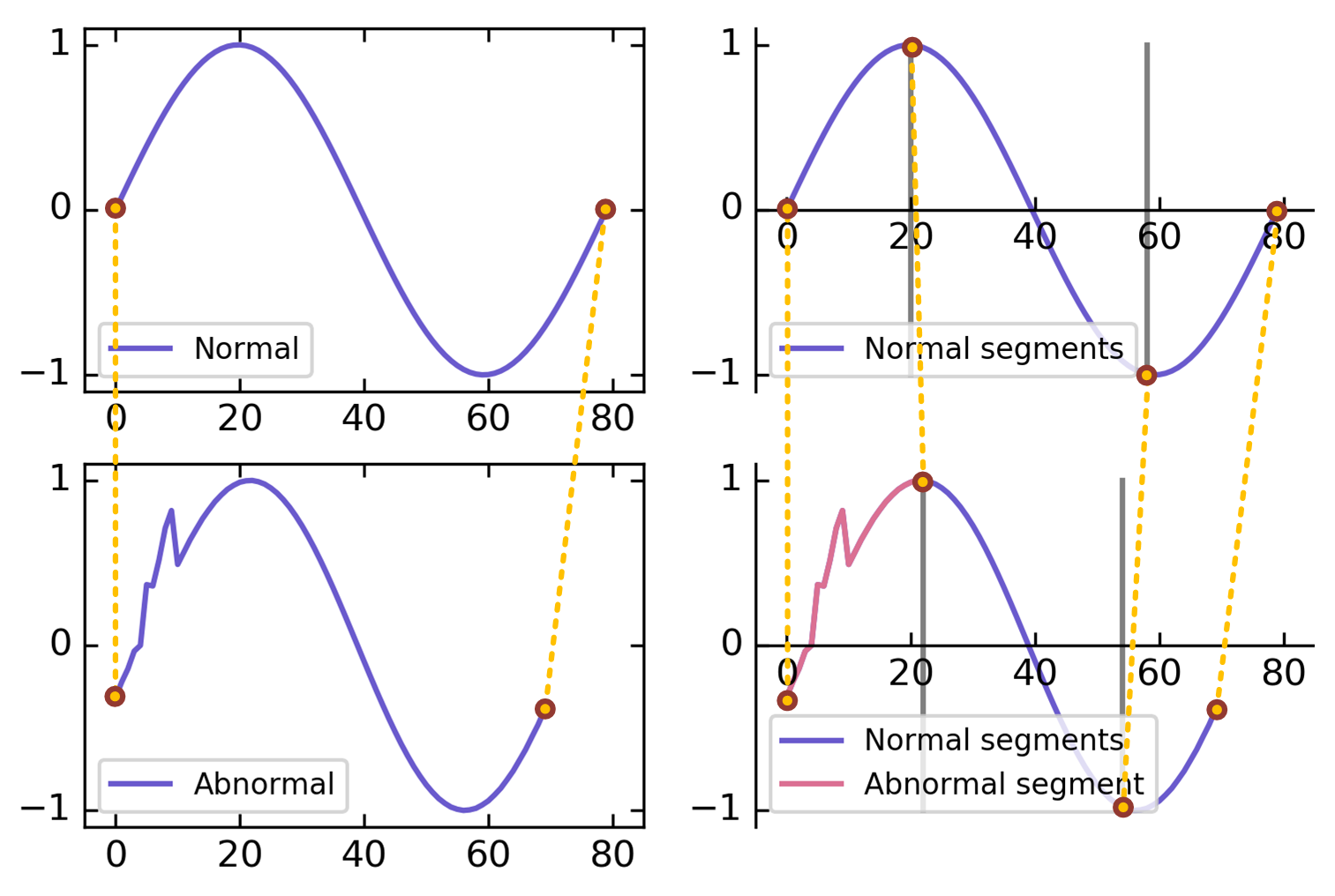}
  \centering
  \caption{Temporal segmentation for normal and anomalous time series.}
  \label{fig:split1}
\end{figure}

\subsection{Temporal Compression Network}

The temporal compression network seeks to acquire a low-dimensional representation of each temporal segment. It is composed of two parts, i.e., a sequence encoder using a gated recurrent unit (GRU) and a decoder with attention mechanism. The low-dimensional representation is given below.

\begin{equation}
\label{eqn:constraint-1}
  \begin{aligned}
      & \bs{h}_c = h(\bs{s},\bs{w}_{en}), \bs{s}' = g_a(\bs{h}_c, \bs{s}_o, \bs{w}_{da}), \\
      & \bs{z}_r = f(\bs{s}, \bs{s}'), \bs{y} = [\bs{h}_c, \bs{z}_r],
  \end{aligned}
\end{equation}
where \(\bs{s}\) is a temporal segment of a time series, that is fed into the sequence encoder and $h(\cdot,\bs{w}_{en})$ represents the sequence encoder with parameters $\bs{w}_{en}$. Likewise, $g_a(\cdot,\bs{w}_{da})$ denotes an attentive sequence decoder with parameters $\bs{w}_{da}$. $\bs{s}_o$ is the attention sequence outputted by the encoder. $\bs{s}'$ represents the reconstruction output from the attentive sequence decoder. $f(\cdot)$ is a function that calculates Euclidean distance and cosine similarity between the input and the reconstruction, and $\bs{z}_r$ is the output of this function. Finally both $\bs{h}_c$ and $\bs{z}_r$ are put into a vector $\bs{y}$ to build a complete representation.

\subsection{GMM Estimation Network}
Given the latent space representation of a temporal segment, the estimation network leverages the framework of GMM to carry out the density estimation. In particular, let $\bs{\mu}_k(i)$ and $\bs{\Sigma}_k(i)$ denote respectively the mean and variance of the $k^{th}$ component in GMM for the $i^{th}$ segment. Recall that $\bs{y}_i$ denote the low-dimensional representation of the $i^{th}$ temporal segment. We assume there is a total of $N_s$ temporal segments in training set. Let $\bs{\gamma}_i = [\gamma_1(i), ..., \gamma_K(i)]^T$ denote a K-dimensional vector, where $\gamma_k(i)$ denote the probability that the $i^{th}$ segment belongs to the $k^{th}$ mixture component, we have

\begin{equation}
\label{eqn:constraint-2}
  \begin{array}{lr}
     \bs{\gamma}_i & = g_e(\bs{y}_i,\bs{w}_{es}),\Phi_k = \sum_{i=1}^{N_s}, \frac{\gamma_{k}(i)}{N_s},\\
     \bs{\mu}_k,\bs{\Sigma}_k & = g^k_g(\{\bs{y}_i, \bs{\gamma}_i\}_{i=1}^{N_s}).
  \end{array}
\end{equation}

$\Phi_k$, $\bs{\mu}_k$, $\bs{\Sigma}_k$ are the mixture weights, the mean, and the covariance of the $k^{th}$ component in the Gaussian mixture model. $g_e(\cdot,\bs{w}_{es})$ represents the estimation network with parameters $\bs{w}_{es}$. $g^k_g(\cdot)$ represents the function that calculates $\bs{\mu}_k,\bs{\Sigma}_k$ from $\{\bs{y}_i, \bs{\gamma}_i\}_{i=1}^{N_s}$, which we detail in Section \ref{subsubsec:GMM_params}, Eq. (\ref{eqn:optimization-GMM}).
Notice that these parameters are calculated in the training process. In addition, the sample energy function $E(\bs{y}_i)$ is given below,

\begin{equation}
\label{eqn:constraint-3}
E(\bs{y}_i)=-\log\left(\sum_{k=1}^K \Phi_k \frac{e^{-\frac{1}{2}[\bs{y}_i-\bs{\mu}_k]^T \bs{\Sigma}_k^{-1}[\bs{y}_i-\bs{\mu}_k]}}{\sqrt{|2\pi\bs{\Sigma}_k|}}\right) .
\end{equation}

It measures the likelihood of a segment sample and will be used to characterize its level of abnormality.

Each beat in the training set in running example 2 is broken into three non-overlapping subsequences, and their latent space distributions are shown in Fig. \ref{fig:split2}, where the three high-density clusters correspond to the latent space representations of the three subsequences of all beats. 
GMM estimation network is further invoked to estimate the density of the above distribution. For the three subsequences of normal beat in Fig. \ref{fig:ecg_ecample}, they have lower anomaly scores because their latent space representations are located in the high-density regions. While three subsequences of abnormal beat are assigned high anomaly scores due to being located in the low-density region, i.e. the red points in Fig. \ref{fig:split2}.

\subsection{Joint Optimization Algorithm}
Seq2GMM jointly optimizes the temporal compression network and the temporal estimation network by making use of a new loss function containing both the reconstruction errors and the likelihood of input segments under the framework of GMM. Recall that $\bs{y}_i$ denote the low-dimensional representation of the $i^{th}$ segment. $\bs{s}(i), \bs{s}'(i)$ represent respectively the $i^{th}$ segment and the reconstructed segment from the attentive decoder. The joint optimization problem for the training of model parameters is given below,

\begin{equation}
  \label{eqn:optimization-1}
  \begin{aligned}
    & \text{min}
    & & \sum_{i=1}^{N_s}D[\bs{s}(i), \bs{s}'(i)] + \lambda\sum_{i=1}^{N_s}E(\bs{y}_i) \\
    & \text{s.t.} & &  (\ref{eqn:constraint-1}),(\ref{eqn:constraint-2}),\\
    & \text{var}  & & \bs{w}_{en},\bs{w}_{da},\bs{w}_{es},\{\bs{\mu}_k,\bs{\Sigma}_k\}.
  \end{aligned}
\end{equation}

The objective function consists of two components. $D(s, s')$ represents the loss function characterizing the reconstruction errors using $L_2$-norm. $E(\bs{y}_i)$ is the sample energy function modeling the probability of observing the input sample. $\lambda$ guides the tradeoff between those two components. The above objective function is highly non-linear and over a large parameter space. It is therefore very challenging to optimize. This is partly because that the sample energy function given in Eq. (\ref{eqn:constraint-3}) is very complex and difficult to optimize. Instead it turns out the objective function in (\ref{eqn:optimization-1}) is convex with respect to $\bs{s}(i), \bs{s}'(i), \{\Phi_k, 1 \leq k \leq K\}$ for fixed $\bs{\mu}_k,\bs{\Sigma}_k$, as shown below.

\begin{prop}
The objective function given in (\ref{eqn:optimization-1}) is a convex function with respect to $\bs{s}(i), \bs{s}'(i), \{\Phi_k, 1 \leq k \leq K\}$ if we fix $\{\bs{\mu}_k,\bs{\Sigma}_k, 1\leq k \leq K\}$ and $\bs{y}_i$.
\end{prop}

\begin{proof}
The reconstruction error is an $L_2$-norm thus it is convex with respect to $\bs{s}(i), \bs{s}'(i)$. In addition, for fixed $\bs{\mu}_k,\bs{\Sigma}_k$, the sample energy function can be reduced to a $-log(\cdot)$ function with respect to the linear combination of $\{\Phi_k, 1 \leq k \leq K\}$ and therefore it is a convex function. Thus the Proposition directly follows from the fact that the sum of two convex functions is also a convex function.
\end{proof}

The above observation motivates us to develop a surrogate-based optimization algorithm in the sequel.

\subsubsection{Initialization via Minimizing the Reconstruction Error}

Instead of directly optimizing Eq. (\ref{eqn:optimization-1}), we solely optimize the reconstruction errors via SGD algorithm. The resulting optimization problem is given by,

\begin{equation}
  \label{eqn:optimization-2}
  \begin{aligned}
    & \text{min}
    & & \sum_{i=1}^{N_s}D[\bs{s}(i), \bs{s}'(i)]  \\
    & \text{s.t.} & &  (\ref{eqn:constraint-1}), \\
    & \text{var}  & & \bs{w}_{en},\bs{w}_{da}.
  \end{aligned}
\end{equation}

The resulting encoder-decoder structure provides an initial set of parameters for the compression network. Please note that while we employ the SGD algorithm in this study, other training algorithms such as adaptive moment estimation (Adam) and root mean square (RMSprop) can also be employed in the training process. After the initialization, we obtain not only an initial set of parameters of the encoder-decoder structure but also the latent space samples $\{\bs{y}_i\}$.
Then an unsupervised clustering algorithm will be invoked to group these latent
space samples.

\subsubsection{Calculate the parameters of GMM via hybrid EM algorithm}
\label{subsubsec:GMM_params}
We first cluster the latent space samples via K-means clustering method, which provides a good starting point for the subsequent GMM training. The clustering result will then be processed via an expectation-maximization algorithm for iterative refinement so that the key parameters of the probabilistic Gaussian mixture model can be better inferred. Given $\{\bs{y}_i, 1 \leq i \leq N_s\}$, the algorithm essentially aims to solve the following optimization problem,

\begin{equation}
  \label{eqn:optimization-GMM}
  \begin{aligned}
    & \text{min}
    & & \sum_{i=1}^{N_s}E(\bs{y}_i) \\
    & \text{var}  & & \Phi_k, \bs{\mu}_k , \bs{\Sigma}_k.
  \end{aligned}
\end{equation}

Please notice that seq2GMM may also encounter the \emph{singularity problem as in GMM}, i.e., if the covariance matrix is singular, trivial solutions are obtained during the training. We address this problem via adding a small positive number e.g., $1e-6$ to the diagonal entries of the covariance matrices to prevent it from happening.

\subsubsection{Update network parameters and membership prediction via SGD}

Once we obtain an initial set of GMM parameters, we further refine the parameters of the compression and estimation network as well as the membership prediction in the sequel. Let $\bs{\bar{\mu}},\bs{\bar{\Sigma}}$ denote the results obtained by solving (\ref{eqn:optimization-GMM}). We further define $E'(\bs{y}_i)$ such that
\begin{equation}
\label{eqn:constraint-5}
E'(\bs{y}_i)=-\log\left(\sum_{k=1}^K \Phi_k \frac{e^{-\frac{1}{2}[\bs{y}_i-\bs{\bar{\mu}}_k]^T \bs{\bar{\Sigma}}_k^{-1}[\bs{y}_i-\bs{\bar{\mu}}_k]}}{\sqrt{|2\pi\bs{\bar{\Sigma}}_k|}}\right) .
\end{equation}

The next optimization problem we seek to solve is given below.

\begin{equation}
  \label{eqn:optimization-4}
  \begin{aligned}
    & \text{min}
    & & \sum_{i=1}^{N_s}D[\bs{s}(i), \bs{s}'(i)] + \lambda\sum_{i=1}^{N_s}E'(\bs{y}_i) \\
    & \text{s.t.} & &  (\ref{eqn:constraint-1}), \bs{\gamma}_i  =  g_e(\bs{y}_i,\bs{w}_{es}), \Phi_k  =  \sum_{i=1}^{N_s} \frac{\gamma_{k}(i)}{N_s}, \\
    & \text{var}  & & \bs{w}_{en},\bs{w}_{da},\bs{w}_{es}.
  \end{aligned}
\end{equation}

The surrogate-based optimization algorithm is summarized in Algorithm \ref{algo:Surrogate}.

\begin{algorithm}
	\caption{Surrogate-Based Optimization}
	\label{algo:Surrogate}
	Initialization by solving the optimization problem given in Eq. (\ref{eqn:optimization-2})
  with SGD algorithm to obtain an initial set of parameters for the encoder-decoder structure\;
	\For {$t = 1$ to $T$} {
		Calculate the initial latent space representation. Then applying the hyrbid
    EM algorithm to optimize Eq.(\ref{eqn:optimization-GMM}) and calculate the parameters of GMM \;
    Compute the parameters of the compression network, the estimation network,
    as well as the membership prediction via applying one epoch of SGD to solve the optimization problem given in Eq. (\ref{eqn:optimization-4})\;
	}
	\Return\{$\bs{w}_{es},\bs{w}_{en},\bs{w}_{da},\bs{\mu},\bs{\Sigma}$\}.
\end{algorithm}

\begin{prop}
Let $o_t$ denote the objective function value obtained after the $t^{th}$ iteration. The sequence generated by Algorithm \ref{algo:Surrogate}, i.e., $\{o_1,o_2,...\}$ converges almost surely. In addition, the sequence monotonically decreases almost surely at each iteration.
\end{prop}

\begin{proof}
It follows from the Robbins-Siegmund theorem that the SGD algorithm converges almost surely to a local optimal solution as long as the learning rate of the SGD decreases at an appropriate rate \cite{SGD-convergence-book}. In addition, the EM algorithm is guaranteed to converge to a local optimal \cite{EM-convergence} under mild assumptions. Therefore, we have $o_{t+1} \leq o_t$ almost surely. Since $o_t$ is bounded from below, the sequence generated by Algorithm \ref{algo:Surrogate} converges almost surely.
\end{proof}

In addition, we have the following proposition that provides both upper and lower bounds to the optimal objective function value of (\ref{eqn:optimization-1}).
\begin{prop}
Let $o^{(1)},o^{(2)}$ denote the optimal objective function values to (\ref{eqn:optimization-2}), (\ref{eqn:optimization-1}) respectively. Let $\tilde{\bs{w}}_{en},\tilde{\bs{w}}_{da},\tilde{\bs{w}}_{es}$ denote the optimal solution to (\ref{eqn:optimization-4}). Let $o^{(3)}$ denote the objective function value obtained by substituting $\tilde{\bs{w}}_{en},\tilde{\bs{w}}_{da},\tilde{\bs{w}}_{es}$ into (\ref{eqn:optimization-1}). We have $o^{(1)} \leq o^{(2)} \leq o^{(3)}$.
\end{prop}

\begin{proof}
Let $\{\bs{s}'^{(1)}(i)\}$ represent the reconstructed segments with minimum reconstruction error via solving (\ref{eqn:optimization-2}). Likewise, $\{\bs{s}'^{(2)}(i)\}$ are the reconstructed segments obtained by solving (\ref{eqn:optimization-1}). It follows that 
\begin{equation}
    \begin{aligned} 
    & \sum_{i=1}^{N_s}D[\bs{s}(i), \bs{s}'^{(1)}(i)] \leq \sum_{i=1}^{N_s}D[\bs{s}(i), \bs{s}'^{(2)}(i)] \\ & \leq \sum_{i=1}^{N_s}D[\bs{s}(i), \bs{s}'^{(2)}(i)] + \lambda\sum_{i=1}^{N_s}E(\bs{y_i}).
    \end{aligned}
\end{equation} 
The first inequality follows from the fact that $\{\bs{s}'^{(1)}(i)\}$ are the reconstructed segments with minimum reconstruction error. The second inequality is due to the fact that the sample energy function is non-negative. Therefore $o^{(1)} \leq o^{(2)}$. Furthermore, since $o^{(2)}$ is the optimal objective function value to (\ref{eqn:optimization-1}), we have $o^{(2)} \leq o^{(3)}$. Thus $o^{(1)} \leq o^{(2)} \leq o^{(3)}$. These bounds can be used to estimate the optimality gap, i.e., how far is the current solution from being optimal.
\end{proof}

\subsection{Data Augmentation}
Data augmentation provides an effective means to increase the diversity of training data by adding modified data from the existing ones. 
In this section, we show how seq2GMM can make use of the data augmentation methods to improve its robustness against the timing errors. For each training sample, we randomly delete 5\% or 10\% of the data points to simulate the signal deletion due to Type-2 timing errors. These synthetic data copies are then added to the training dataset to optimize the seq2GMM model.

\subsection{Recurrent Anomaly Shapelet Localization}

A detected anomaly shapelet could provide rich semantic information for users to understand the reason that such a time series is deemed as being anomalous. Please note that as opposed to the traditional shapelet discovery approaches, the recurrent anomaly shapelets as well as the associated anomaly scores are all calculated in the latent space through a recurrent encoder-decoder architecture.

\section{Experiment Results}\label{section:5}
Extensive experiments have been conducted to assess the performance of the proposed seq2GMM framework. We first evaluate its performance on a synthetic dataset to illustrate how it remains effective robustly for rhythmic time series with timing errors. 
For clarity, we generate a $\sin(\cdot)$ signal and obtain multiple copies via cyclically shifting the original time series to mimic Type-1 timing errors. We also generate anomalous samples via injecting synthetic anomalies into the original $\sin(\cdot)$ signal. Fig. \ref{fig:synthetic} illustrates the latent space representation of the normal time series as well as the synthetic anomalies. It is evident that via learning from multiple time series simultaneously, seq2GMM can successfully group non-anomalous time series together in the latent space even if there exist timing drifts between them. In addition, the latent space representation of the anomalous time series will deviate significantly from non-anomalous ones even if the anomalous signal spans only a short period in the time domain.

\begin{figure}[!tbp]
  \centering
  \subfigure[Cyclically shifted normal time series and anomalous time series]{
    \centering
    \includegraphics[height=1in,width=1.6in]{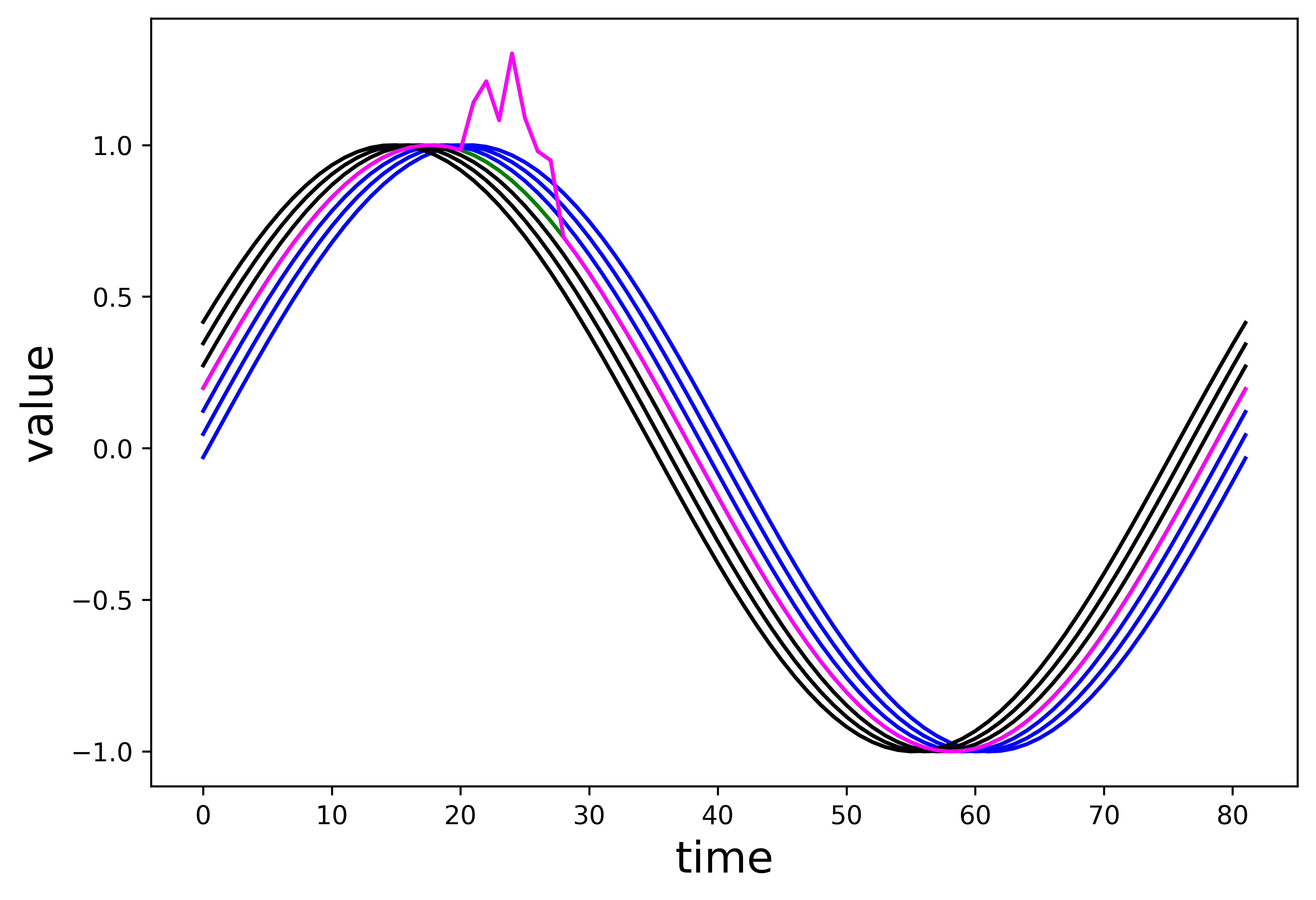}
  }\ 
  \subfigure[Latent space representation of shifted normal time series and anomalous time series.]{
    \centering
    \includegraphics[height=1in,width=1.6in]{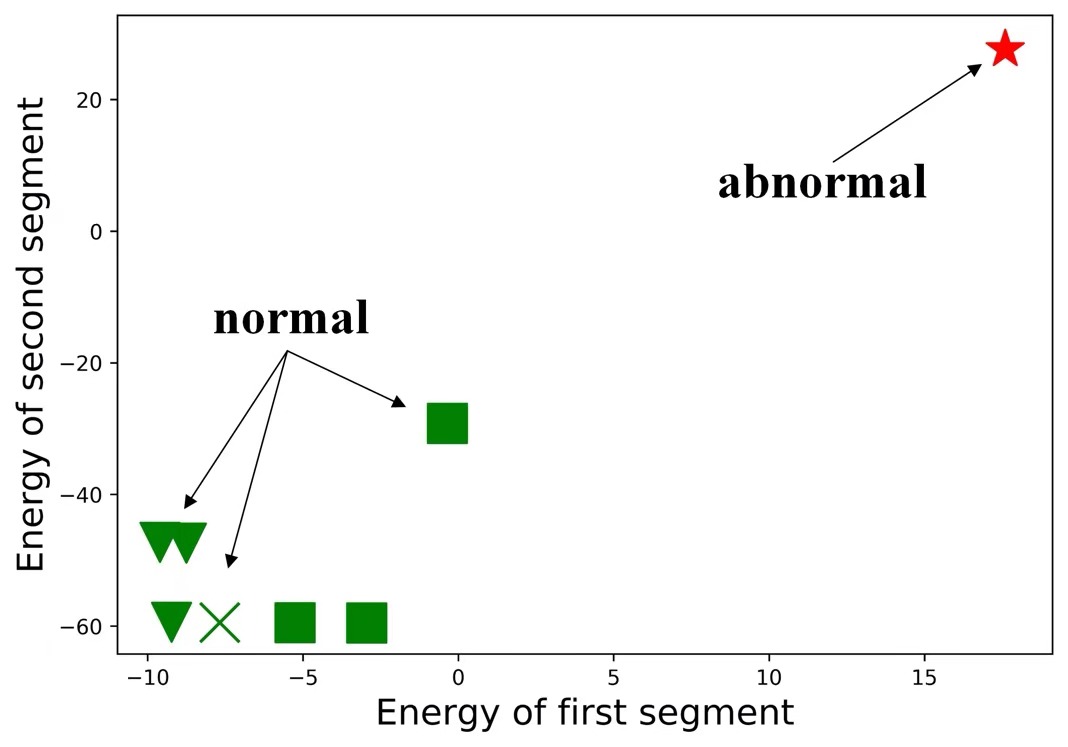}
  }%
  \quad
  \subfigure[Normal time series and anomalous time series with timing errors.]{
    \centering
    \includegraphics[height=1in,width=1.6in]{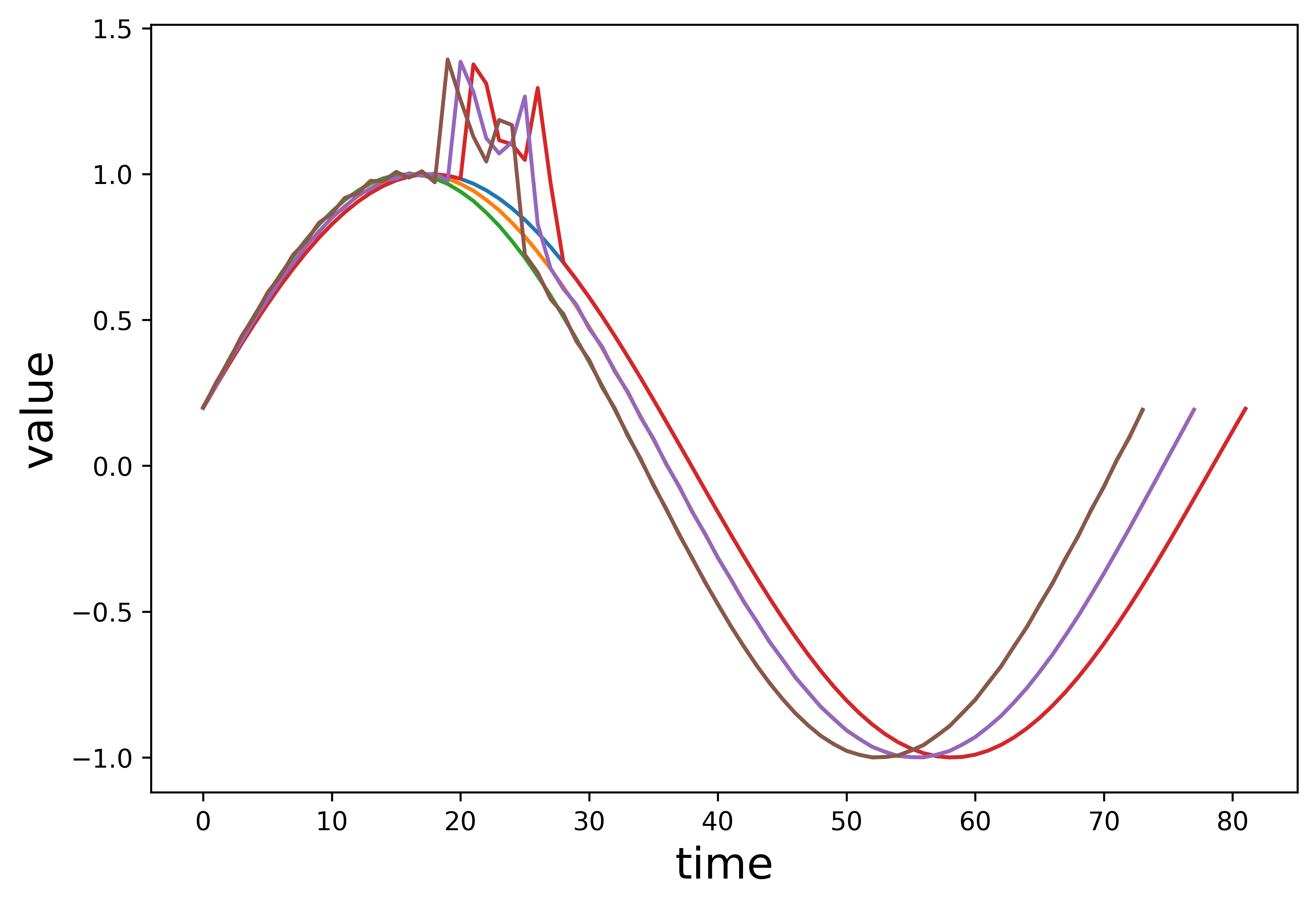}
  }\ 
  \subfigure[Latent space representations of normal time series and anomalous time series with timing errors.]{
    \centering
    \includegraphics[height=1in,width=1.6in]{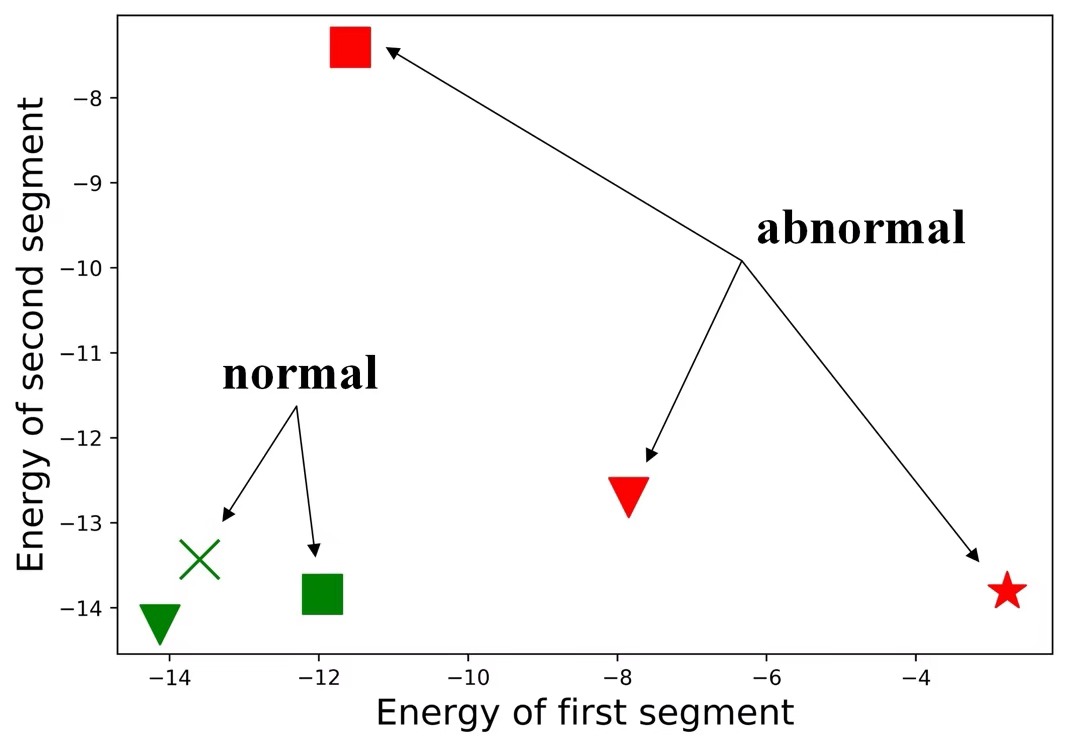}
  }%
  \centering
  \caption{Representative time series are given in (a) and (c).
In figure (b), the cross represents the original time series. The red star denotes the latent space representation of the anomalous time series. The blue squares represent the right-shifted normal time series while the green inverted triangles are the left-shifted time series. In figure (d), the cross represents the original time series. The red star, inverted triangle and square represent the anomalous time series with different degrees of timing errors. The green inverted triangle and square denote the latent space representation of the normal time series with different degrees of timing errors.}
  \label{fig:synthetic}
\end{figure}

The proposed model has been further evaluated on a multitude of benchmark datasets from UCR time series, including the TwoLeadECG, ProximalPhalanxTW, FacesUCR, and MedicalImage datasets, as well as the CMU motion capture database. 
Please note that TwoLead ECG is particularly useful for diagnosing rhythm disturbances, which often happen for patients on a coronary care or during operations. 
CMU motion capture database contains motion capture records, i.e., walking, running, jumping, of different human body postures. The data are collected from multiple sensors, and the output of each sensor is a univariate time series. In the experiment, walking motion capture records are viewed as normal time series while the others are deemed as anomalies. The training set size is 910, the test set size is 236. 
Multivariate time series from a total of four sensors are chosen for the experiments and anomaly detection is performed in each dimension, i.e., CMU-1 to CMU-4, respectively.
As an alternative way, multi-dimensional recurrent neural network \cite{graves2007multi} can be employed to process multivariate data and learn the latent space distribution directly.
Please note that the aforementioned datasets fall into different application domains. In particular, TwoLeadECG dataset falls into the domain of health informatics; CMU-1 to CMU-4 datasets can be used for IoT data analysis (motion analysis); ProximalPhalanxTW, FacesUCR, and MedicalImage datasets can be used for shape analysis.

In particular, we conducted experiments aiming to answer the following questions.
\begin{itemize}
\item Q1. \textbf{Accuracy}: What level of accuracy seq2GMM can achieve compared with other state-of-the-art baseline algorithms? Will the performance degrade if there are anomalies in the training dataset?
\item Q2. \textbf{Timing errors}: Can seq2GMM model be employed to detect anomalies in time series with \emph{timing errors}?
\item Q3. \textbf{Localization}: Does seq2GMM offer a visualization and localization interface that helps human experts localize the anomaly shapelet and understand why a rhythmic time series is deemed as an anomaly?
\end{itemize}

\subsection{Accuracy of Anomaly Detection}
We follow the method described in \cite{4781172} to create the datasets for study from the UCR time series classification datasets. Each UCR time series classification dataset contains a number of classes and one of them will be chosen as the non-anomalous major class. We then randomly sample the remaining data and inject these anomalies into the major class to construct the final dataset for the experiment. The standard AUC score (area under the receiver operator characteristic curve) and standard AUPR score (area under the precision recall curve) are used for characterizing the anomaly detection accuracy.
For each dataset, the AUC score and AUPR score are averaged over twenty runs.  Ten state-of-the-art anomaly detection techniques are compared with the proposed seq2GMM method as baseline methods, including LOF \cite{breunig2000lof}, DTW (dynamic time warping),  DTW+LoF, Parzen Windows \cite{yeung2002parzen}, DL (dictionary learning) OCSVM \cite{bevilacqua2015dictionary}, FD (frequency domain) OCSVM \cite{scholkopf1999support}, Matrix Profile \cite{yeh2016matrix}, LSTM-ED \cite{malhotra2016lstm}, DAGMM \cite{zong2018deep}, BeatGAN \cite{BeatGAN}.

The experiment is conducted on a computer server with four NVIDIA GTX1080, and we use TensorFlow to implement the deep learning model. The model \emph{hyperparameters} including the size of the hidden layer of the sequence autoencoder, the total number of temporal segments $M$, and the total number of mixture components $K$ in GMM are optimized by a validation dataset.

It is seen from Table \ref{tab:AUC} that the proposed unsupervised seq2GMM outperforms all other anomaly detection methods on the four benchmark datasets. In particular, for the TLECG datasets, it is seen that seq2GMM can provide around 6\% performance gain over the second-best algorithm. DAGMM and LSTM-ED achieve respectively the second-best result on two benchmark datasets. The AUPR scores are provided in  Table \ref{tab:AUPR}. These results reveal the superior performance of seq2GMM and underscore the significant benefits of combining sequence attentive autoencoder-decoder architecture with GMM for anomaly detection.

\begin{table*}[!t]
\centering
\caption{Average AUC scores of seq2GMM and the baseline methods. The best and the second-best results are shown in bold and underline, respectively.}\label{tab:AUC}
\scalebox{0.9}{
\begin{tabular}{@{}ccccccccccc@{}}
\toprule
               & CMU-1          & CMU-2          & CMU-3           & CMU-4           & TLECG          & FacesUCR        & MedicalImages  & PPTW           & Average rank  \\ \midrule
DTW            & 95.55          & 87.80          & 92.96           & 97.34           & 83.82          & 93.37           & 46.11          & 92.06          & 6.50          \\
LOF            & 77.49          & 80.55          & 83.76           & 85.24           & 70.17          & 91.70           & 66.89          & 89.09          & 7.75          \\
DTW+LoF        & \textbf{99.87} & {\ul 99.51}    & {\ul 99.89}     & \textbf{100.00} & 70.09          & 98.90           & 66.51          & 85.71          & 4.00          \\
DL-OCSVM       & 58.37          & 65.12          & 99.77           & \textbf{100.00} & 64.92          & 67.95           & 54.65          & 93.65          & 6.88          \\
FD-OCSVM       & 96.23          & 87.22          & 98.60           & \textbf{100.00} & 87.68          & 96.92           & 51.27          & 91.47          & 4.75          \\
ParzenW        & 95.12          & 83.95          & 90.70           & 94.41           & 80.72          & 91.67           & \textbf{73.01} & 89.48          & 6.75          \\
Matrix Profile & 52.38          & 52.38          & 54.76           & 50.00           & 50.89          & 50.00           & 50.00          & 64.29          & 10.88         \\
LSTM-ED        & 95.28          & 94.57          & 95.50           & \textbf{100.00} & 84.93          & 96.13           & 66.81          & 94.24          & 4.13          \\
DAGMM          & 94.29          & 93.51          & 68.83           & 64.72           & 58.45          & 83.90           & 64.73          & 90.34          & 8.13          \\
BeatGan        & 96.17          & 97.76          & \textbf{100.00} & \textbf{100.00} & {\ul 87.75}    & {\ul 99.92}     & 63.18          & {\ul 95.03}    & {\ul 2.75}    \\
Seq2GMM        & {\ul 99.71}    & \textbf{99.71} & 99.71           & {\ul 99.78}     & \textbf{93.23} & \textbf{100.00} & {\ul 67.34}    & \textbf{98.08} & \textbf{2.25} \\ \bottomrule
\end{tabular}}
\end{table*}

\begin{table*}[!t]
\centering
\caption{Average AUPR scores of seq2GMM and the baseline methods. The best and the second-best result are shown in bold and underline, respectively.}\label{tab:AUPR}
\scalebox{0.9}{
\begin{tabular}{@{}ccccccccccc@{}}
\toprule
               & CMU-1          & CMU-2          & CMU-3          & CMU-4           & TLECG          & FacesUCR        & MedicalImages  & PPTW           & Average rank  \\ \midrule
DTW            & 90.66          & 87.06          & 83.50          & 96.03           & 36.29          & 65.64           & 13.90          & 87.67          & 6.13          \\
LOF            & 49.62          & 66.25          & 67.22          & 63.30           & 30.19          & 31.27           & 16.69          & 87.01          & 9.00          \\
DTW+LoF        & \textbf{98.92} & \textbf{97.53} & \textbf{98.95} & \textbf{100.00} & 23.98          & 87.02           & 21.94          & 86.90          & 4.13          \\
DL-OCSVM       & 91.00          & 87.04          & 84.86          & \textbf{100.00} & 30.02          & 63.04           & 14.42          & 87.42          & 5.88          \\
FD-OCSVM       & 93.85          & 86.78          & {\ul 94.71}    & \textbf{100.00} & 51.87          & 62.54           & 13.44          & 87.59          & 5.25          \\
ParzenW        & 75.00          & 61.67          & 67.21          & 73.33           & 44.63          & 63.94           & 27.39          & 87.18          & 7.38          \\
Matrix Profile & 56.62          & 56.62          & 58.78          & 54.45           & 55.29          & 54.49           & \textbf{54.53} & 67.45          & 8.25          \\
LSTM-ED        & 94.05          & 89.57          & 84.30          & \textbf{100.00} & {\ul 72.97}    & 67.07           & 20.58          & 88.97          & {\ul 3.38}    \\
DAGMM          & 86.41          & 55.45          & 19.51          & 15.34           & 53.89          & 61.32           & 23.38          & 87.45          & 8.00          \\
BeatGan        & 77.29          & 70.31          & 76.66          & {\ul 96.21}     & 64.01          & {\ul 99.31}     & 11.54          & {\ul 93.59}    & 5.75          \\
Seq2GMM        & {\ul 96.39}    & {\ul 92.22}    & 89.88          & \textbf{100.00} & \textbf{81.68} & \textbf{100.00} & {\ul 35.79}    & \textbf{98.09} & \textbf{1.63} \\ \bottomrule
\end{tabular}}
\end{table*}

\subsection{Model Training with Contaminated Data}
We also assess the performance of the proposed framework in the presence of contaminated training data via artificially injecting anomalies into the training dataset, as shown in Table \ref{tab:percentage}. It is seen that the seq2GMM remains robust against anomalies in the training dataset and only degrades slightly even when a total of 10\% of the training set number of anomalies have been injected into the training set.

\begin{table}[!t]
  \centering
  \caption{AUC scores of seq2GMM when training data is contaminated. The numbers in brackets after the dataset name and AUC scores are the size of original training set and the number of injected anomaly samples, respectively.} \label{tab:percentage}
  \scalebox{0.9}{
  \begin{tabular}{cccc}
  \toprule
  & MedicalImg (203) & PPTW (180)& TLECG (12)\\
  \midrule
  & 67.34 (0)& 98.08 (0)& 93.23 (0)\\
  & 66.75 (10)& 94.84 (9)& 91.73 (1)\\
  & 66.36 (20)& 94.24 (18)& 90.21 (2)\\
  \midrule
  loss &0.98\%& 3.24\%& 3.02\% \\
  \bottomrule
  \end{tabular}
  }
  \centering

  \end{table}

\subsection{Model Training with Augmented Dataset}

As shown in Table \ref{tab:variable-length}, seq2GMM performs well robustly and degrade only marginally in the presence of Type-2 timing errors, i.e. a total of 10\% sampling data points are removed. As a comparison, the performance of BeatGAN degrades considerably when the sequence length varies.

The results in Table \ref{tab:AUC}, Table \ref{tab:AUPR} and Table \ref{tab:variable-length} indicate the robustness of seq2GMM against timing errors.
The quasi-periodic time series containing timing errors can be viewed as adversarial samples \cite{ryan2019pattern}.
It is well known that adversarial training with adversarial samples helps to improve the model robustness. In other words, the obtained machine learning model remain robust in the presence of timing errors.
Secondly, Seq2GMM divides the sequence into multiple segments and train a neural network to learn these segments jointly. Since various segments are linked with different temporal trends and dynamics, learning them jointly can be viewed as a multi-task learning approach, which can effectively strengthen the adversarial robustness of the obtained machine learning model \cite{mao2020multitask}.

The above analysis inspires us to improve the robustness of seq2GMM by augmenting the training data set. As shown in Table \ref{tab:variable-length}, compared with seq2GMM, seq2GMM$^{aug}$ has a significant performance improvement on all test sets, and the performance degradation with timing errors is significantly smaller than that of seq2GMM.

\begin{table}[!t]
  \centering
  \caption{AUC scores for different sequence lengths of two lead ECG data}\label{tab:variable-length}
  \scalebox{0.9}{
  \begin{tabular}{lccc}
    \toprule
    Ratio          & 100\%            & 95\%            & 90\%           \\ \midrule
    DTW            & 83.82          & 82.34          & 81.35          \\
    LOF            & 70.17          & 70.26          & 70.24          \\
    DTW+LoF        & 70.09          & 65.04          & 54.64          \\
    DL-OCSVM       & 64.92          & 71.99          & 73.53          \\
    FD-OCSVM       & 87.68          & 73.23          & 58.06          \\
    ParzenW        & 80.72          & 70.81          & 67.64          \\
    Matrix Profile & 50.89          & 50.89          & 50.89          \\
    LSTM-ED        & 84.93          & 80.98          & 77.96          \\
    DAGMM          & 58.45          & 51.37          & 50.18          \\
    BeatGAN        & 87.75          & 78.87          & 71.98          \\
    Seq2GMM        & {\ul 93.23}    & {\ul 88.39}    & {\ul 87.81}    \\ 
    Seq2GMM$^{aug}$& \textbf{97.11} & \textbf{97.45} & \textbf{95.95} \\ \bottomrule
  \end{tabular}}
  \end{table}

We also conduct anomaly detection experiments on time series containing Type-2 timing errors in which the length of the train time series remains the same as those of the test time series. More concretely, we randomly drop 10\% of the data points from each of the time series to obtain a new dataset. We then split this dataset into training and testing dataset to assess the anomaly detection performance of the proposed seq2GMM framework. As shown in Table \ref{tab:irregular}, the proposed seq2GMM model outperforms other state-of-the-art methods by a large margin.

\begin{table*}[!t]
\centering
\caption{AUC scores for time series with Type}\label{tab:irregular}
\scalebox{0.9}{
\begin{tabular}{@{}ccccccccccc@{}}
\toprule
               & CMU-1                & CMU-2          & CMU-3           & CMU-4           & TLECG          & FacesUCR       & MedicalImages  & PPTW           & average rank  \\ \midrule
DTW            & 95.61                & 87.69          & 92.85           & 97.85           & 85.67          & 86.87          & 45.24          & 94.66          & 5.75          \\
LOF            & 77.05                & 76.94          & 82.70           & 85.09           & 69.49          & 81.56          & 66.87          & 62.64          & 8.00          \\
DTW+LoF        & \textbf{99.76} & {\ul 98.43}    & \textbf{100.00} & \textbf{100.00} & 66.62          & {\ul 95.92}    & 59.37          & {\ul 94.69}    & {\ul 2.75}    \\
DL-OCSVM       & 64.88       & 63.95          & 96.74           & 98.83           & 65.11          & 63.09          & 47.85          & 68.18          & 8.38          \\
FD-OCSVM       & 96.79                & 87.95          & 98.41           & \textbf{100.00} & 83.31          & 85.39          & 51.26          & 85.34          & 5.00          \\
ParzenW        & 87.90                & 84.41          & 85.58           & 88.60           & 69.94          & 85.40          & \textbf{72.04} & 92.29          & 6.25          \\
Matrix Profile & 54.76                & 52.38          & 54.76           & 49.77           & 49.91          & 49.83          & 51.28          & 52.22          & 10.50         \\
LSTM-ED        & 94.86                & 94.51          & 92.89           & \textbf{100.00} & 83.23          & 89.48          & 64.15          & 93.64          & 4.13          \\
DAGMM          & 83.96                & 90.23          & 67.24           & 52.56           & 60.59          & 62.44          & 59.01          & 92.59          & 8.00          \\
BeatGan        & 89.92                & 95.25          & {\ul 99.99}     & {\ul 99.99}     & {\ul 85.96}    & 91.91          & 44.47          & 91.67          & 4.75          \\
Seq2GMM        & {\ul 99.69}          & \textbf{99.60} & 97.54           & 99.60           & \textbf{89.14} & \textbf{97.63} & {\ul 67.09}    & \textbf{97.71} & \textbf{2.13} \\ \bottomrule
\end{tabular}}
\end{table*}

\subsection{Anomaly Visualization and Anomaly Shapelet Localization}

\begin{figure}[!tbp]
  \centering
    \centering
    \includegraphics[width=2in]{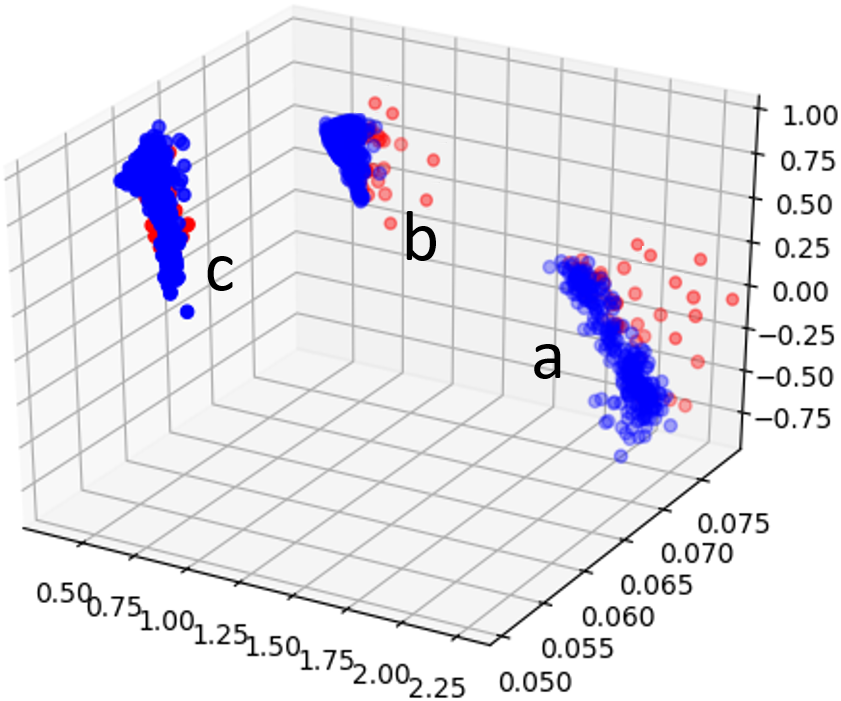}
  \centering
  \caption{
  Latent space representation of temporal segments. The blue and red points in
  latent space correspond to the anomalous and non-anomalous time segments respectively.}
  \label{fig:split2}
\end{figure}

The obtained anomalies can be visualized and analyzed in the latent space via the proposed encoder-decoder structure, as shown in Fig. \ref{fig:split2}.
In the latent space, temporal segments are clustered and those far away from the center of the associated cluster are deemed as anomalies.
We can obtain an anomaly score for the latent space representation of each temporal segment. Upon obtaining the anomaly score, we can then pinpoint the \emph{recurrent anomaly shapelets} in the original space to help human experts understand why this time series is deemed as an anomaly, as shown in Fig. \ref{fig:analysis}. Therefore, seq2GMM provides \emph{local interpretability}.

\begin{figure*}[t]
  \centering
    \centering
    \includegraphics[width=6.5in]{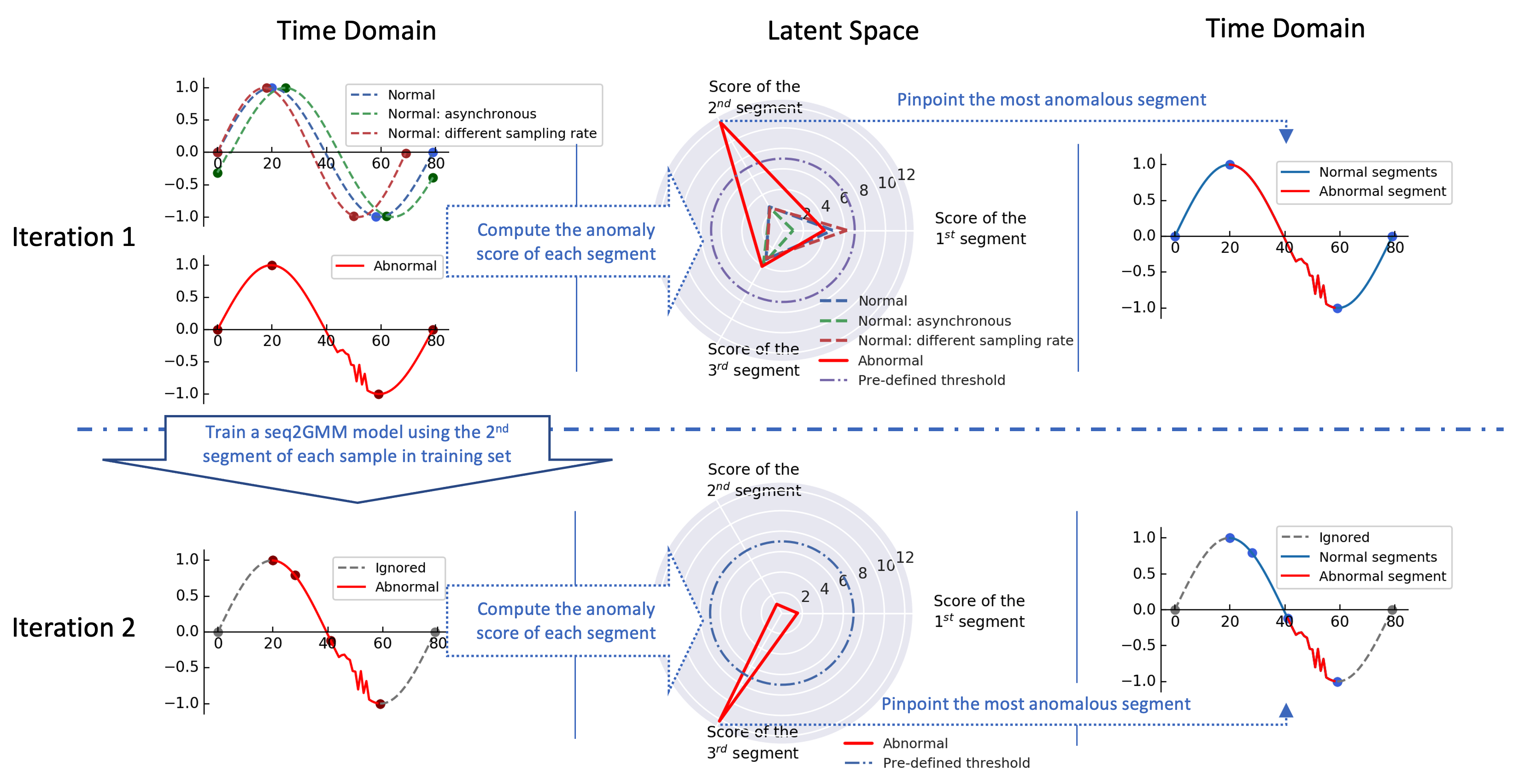}
  \centering
  \caption{Anomaly visualization and anomaly shapelet localization via analysis in latent space.}
  \label{fig:analysis}
\end{figure*}

\subsection{The Convergence Behavior}

\begin{figure}[htbp]
  \centering
  \includegraphics[width=3in]{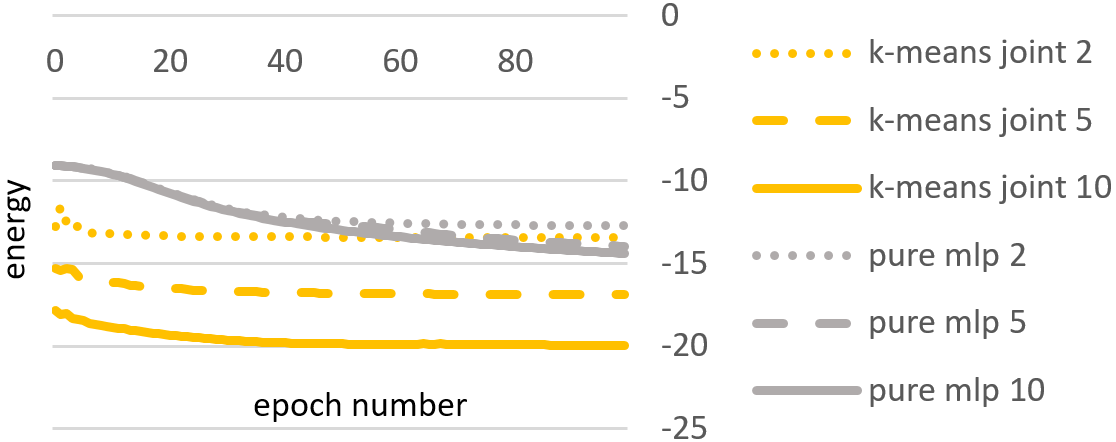}
  \caption{The convergence behaviors of the proposed surrogate-based training algorithm for different models, i.e., joint clustering with estimation network and the estimation network (MLP), in which 2, 5 and 10 denote respectively the number of components in GMM.}
  \label{fig:kmeans}
\end{figure}

Fig. \ref{fig:kmeans} compares the convergence behaviors of the proposed surrogate-based training algorithm with that of the traditional SGD algorithm. The proposed surrogate-based training strategy combines both the SGD and EM algorithm. It is seen that it can quickly converge to a set of parameters with low sample energy. On the other hand, the traditional SGD algorithm reaches an error floor after a few epochs.

\subsection{Ablation Study}

We conduct an ablation study to elucidate the role of time series segmentation in the proposed framework. Please note that we fix all hyper-parameters except the number of time series segments $M$ during the experiment. Table \ref{tab:ablation1} lists the average AUC scores after multiple runs over different segments settings. It is seen that the time series segmentation can effectively improve the performance of anomaly detection.

\begin{table}[!t]
\centering
\caption{AUC scores of seq2GMM for different segments of two lead ECG data}\label{tab:ablation1}
\scalebox{0.9}{
\begin{tabular}{@{}cc@{}}
\toprule
\# of   segments & AUC            \\ \midrule
w.o.              & 71.62          \\
2                & 78.70          \\
3                & {\ul 86.27}    \\
4                & \textbf{93.05} \\ \bottomrule
\end{tabular}}
\end{table}

\subsection{Scalability Analysis}

Finally, we theoretically analyze the computational and space complexity of the proposed machine learning model to support scalability analysis.

Let $T$ and $D$ represent the length and dimension of time series, respectively. $D_L$, $D_E$ and $K$ denote respectively the dimensions of latent space, the number of neurons in middle layer of estimation netowork, and the number of Gaussian mixture components. 
The computational complexity of seq2GMM is $O(T D C(D_L) + T C(D_L) + C(D_L, D_E, K))$, and the space complexity is $O(D C(D_L) + D T + C (D_L, D_E, K))$, where $C(\cdot)$ represents a constant controlled by the network hyperparameters.

It is evident that both the computational complexity and space complexity of seq2GMM will scale only linearly with $D$ and $T$. Please note that in practice, a variety of RNN acceleration algorithms \cite{dong2020rtmobile,wang2018c} can also be employed to further reduce the computational complexity of the seq2GMM.
In addition, when the time series is segmented, according to the aforementioned analysis, the complexity will decrease with the length of the temporal segment.

\section{Conclusion}\label{section:6}
We have developed an unsupervised sequence to GMM learning framework for anomaly detection in quasi-periodic time series with timing errors. We have also developed a surrogate-based optimization algorithm and a data augmentation method to train this deep learning model. The proposed method can also be used to pinpoint the anomaly shapelet.
Extensive experiments have been conducted to showcase the superior performance of the proposed seq2GMM framework and underscore the significant benefits of combining the deep learning architecture with a Gaussian mixture model for detecting anomalies in quasi-periodic time series.




\bibliographystyle{IEEEtran}
\bibliography{IEEEabrv,ref}

\begin{thebibliography}{10}
\providecommand{\url}[1]{#1}
\csname url@samestyle\endcsname
\providecommand{\newblock}{\relax}
\providecommand{\bibinfo}[2]{#2}
\providecommand{\BIBentrySTDinterwordspacing}{\spaceskip=0pt\relax}
\providecommand{\BIBentryALTinterwordstretchfactor}{4}
\providecommand{\BIBentryALTinterwordspacing}{\spaceskip=\fontdimen2\font plus
\BIBentryALTinterwordstretchfactor\fontdimen3\font minus
  \fontdimen4\font\relax}
\providecommand{\BIBforeignlanguage}[2]{{%
\expandafter\ifx\csname l@#1\endcsname\relax
\typeout{** WARNING: IEEEtran.bst: No hyphenation pattern has been}%
\typeout{** loaded for the language `#1'. Using the pattern for}%
\typeout{** the default language instead.}%
\else
\language=\csname l@#1\endcsname
\fi
#2}}
\providecommand{\BIBdecl}{\relax}
\BIBdecl

\bibitem{ECG-nature-medicine}
A.~Y. Hannun, P.~R.~M. Haghpanahi, G.~H. Tison, C.~Bourn, M.~P. Turakhia, and
  A.~Y. Ng, ``Cardiologist-level arrhythmia detection and classification in
  ambulatory electrocardiograms using a deep neural network,'' \emph{Nature
  Medicine}, vol.~25, no.~1, pp. 65--69, 2019.

\bibitem{Basu2007}
S.~Basu and M.~Meckesheimer, ``{Automatic outlier detection for time series: an
  application to sensor data},'' \emph{Knowledge and Information Systems},
  vol.~11, no.~2, pp. 137--154, 2007.

\bibitem{Xu2018}
H.~Xu, W.~Chen, N.~Zhao, Z.~Li, J.~Bu, Z.~Li, Y.~Liu, Y.~Zhao, D.~Pei, Y.~Feng
  \emph{et~al.}, ``Unsupervised anomaly detection via variational auto-encoder
  for seasonal {KPIs} in web applications,'' in \emph{Proceedings of the 2018
  World Wide Web Conference on World Wide Web}, 2018, pp. 187--196.

\bibitem{yang2016deep}
K.~Yang, R.~Liu, Y.~Sun, J.~Yang, and X.~Chen, ``{Deep network analyzer (DNA):
  A} big data analytics platform for cellular networks,'' \emph{IEEE Internet
  of Things Journal}, vol.~4, no.~6, pp. 2019--2027, 2016.

\bibitem{Steinwart:2005:CFA:1046920.1058109}
I.~Steinwart, D.~Hush, and C.~Scovel, ``A classification framework for anomaly
  detection,'' \emph{Journal of Machine Learning Research}, vol.~6, no.~2, pp.
  211--232, 2005.

\bibitem{Amer2013}
M.~Amer, M.~Goldstein, and S.~Abdennadher, ``{Enhancing one-class support
  vector machines for unsupervised anomaly detection},'' in \emph{ACM SIGKDD
  Workshop on Outlier Detection and Description}.\hskip 1em plus 0.5em minus
  0.4em\relax ACM Press, 2013, pp. 8--15.

\bibitem{breunig2000lof}
M.~M. Breunig, H.-P. Kriegel, R.~T. Ng, and J.~Sander, ``{LOF: I}dentifying
  density-based local outliers,'' in \emph{ACM Sigmod Record}, 2000, pp.
  93--104.

\bibitem{zong2018deep}
B.~Zong, Q.~Song, M.~R. Min, W.~Cheng, C.~Lumezanu, D.~Cho, and H.~Chen, ``Deep
  autoencoding gaussian mixture model for unsupervised anomaly detection,'' in
  \emph{International Conference on Learning Representations}, 2018.

\bibitem{motifDiscovery}
K.~Uehara, ``Discovery of time-series motif from multi-dimensional data based
  on {MDL} principle,'' \emph{Machine Learning}, vol.~58, no. 2/3, pp.
  p.269--300, 2005.

\bibitem{BeatGAN}
B.~Zhou, S.~Liu1, B.~Hooi, X.~Cheng, and J.~Ye, ``{BeatGAN: A}nomalous rhythm
  detection using adversarially generated time series,'' in \emph{Proceedings
  of the Twenty-Eighth International Joint Conference on Artificial
  Intelligence (IJCAI-19)}, 2019, pp. 4433--4439.

\bibitem{jin2000digital}
H.~Jin and E.~K. Lee, ``A digital-background calibration technique for
  minimizing timing-error effects in time-interleaved {ADCs},'' \emph{IEEE
  Transactions on Circuits and Systems II: Analog and Digital Signal
  Processing}, vol.~47, no.~7, pp. 603--613, 2000.

\bibitem{wolaver1991phase}
D.~H. Wolaver and D.~H. Wolaver, \emph{Phase-locked loop circuit design}.\hskip
  1em plus 0.5em minus 0.4em\relax Prentice Hall Hoboken, NJ, USA, 1991,
  vol.~29.

\bibitem{yang2020decoding}
K.~Yang, J.~Ren, C.~Tian, J.~Wang, and H.~V. Poor, ``Decoding binary linear
  codes over channels with synchronization errors,'' \emph{IEEE Journal on
  Selected Areas in Communications}, vol.~38, no.~12, pp. 2853--2863, 2020.

\bibitem{mitzenmacher2009survey}
M.~Mitzenmacher, ``A survey of results for deletion channels and related
  synchronization channels,'' \emph{Probability Surveys}, vol.~6, pp. 1--33,
  2009.

\bibitem{ye2009time}
L.~Ye and E.~Keogh, ``{Time series shapelets: A} new primitive for data
  mining,'' in \emph{Proceedings of the 15th ACM SIGKDD international
  conference on Knowledge discovery and data mining}, 2009, pp. 947--956.

\bibitem{beggel2019time}
L.~Beggel, B.~X. Kausler, M.~Schiegg, M.~Pfeiffer, and B.~Bischl, ``Time series
  anomaly detection based on shapelet learning,'' \emph{Computational
  Statistics}, vol.~34, no.~3, pp. 945--976, 2019.

\bibitem{yang2018active}
K.~Yang, J.~Ren, Y.~Zhu, and W.~Zhang, ``Active learning for wireless {IoT}
  intrusion detection,'' \emph{IEEE Wireless Communications}, vol.~25, no.~6,
  pp. 19--25, 2018.

\bibitem{chandola2009anomaly}
V.~Chandola, A.~Banerjee, and V.~Kumar, ``Anomaly detection: A survey,''
  \emph{ACM computing surveys (CSUR)}, vol.~41, no.~3, p.~15, 2009.

\bibitem{hundman2018detecting}
K.~Hundman, V.~Constantinou, C.~Laporte, I.~Colwell, and T.~Soderstrom,
  ``Detecting spacecraft anomalies using {LSTMs} and nonparametric dynamic
  thresholding,'' in \emph{Proceedings of the 24th ACM SIGKDD international
  conference on knowledge discovery \& data mining}, 2018, pp. 387--395.

\bibitem{dou2019pc}
S.~Dou, K.~Yang, and H.~V. Poor, ``{PC$^{2}$A}: Predicting collective
  contextual anomalies via {LSTM} with deep generative model,'' \emph{IEEE
  Internet of Things Journal}, vol.~6, no.~6, pp. 9645--9655, 2019.

\bibitem{ruff2018deep}
L.~Ruff, R.~Vandermeulen, N.~Goernitz, L.~Deecke, S.~A. Siddiqui, A.~Binder,
  E.~M{\"u}ller, and M.~Kloft, ``Deep one-class classification,'' in
  \emph{International conference on machine learning}.\hskip 1em plus 0.5em
  minus 0.4em\relax PMLR, 2018, pp. 4393--4402.

\bibitem{erfani2016high}
S.~M. Erfani, S.~Rajasegarar, S.~Karunasekera, and C.~Leckie,
  ``High-dimensional and large-scale anomaly detection using a linear one-class
  {SVM} with deep learning,'' \emph{Pattern Recognition}, vol.~58, pp.
  121--134, 2016.

\bibitem{shen2020timeseries}
L.~Shen, Z.~Li, and J.~Kwok, ``Timeseries anomaly detection using temporal
  hierarchical one-class network,'' \emph{Advances in Neural Information
  Processing Systems}, vol.~33, pp. 13\,016--13\,026, 2020.

\bibitem{golub2013matrix}
G.~H. Golub and C.~F. Van~Loan, \emph{Matrix computations}.\hskip 1em plus
  0.5em minus 0.4em\relax JHU press, 2013.

\bibitem{percival2000wavelet}
D.~B. Percival and A.~T. Walden, \emph{Wavelet methods for time series
  analysis}.\hskip 1em plus 0.5em minus 0.4em\relax Cambridge university press,
  2000, vol.~4.

\bibitem{faloutsos1994fast}
C.~Faloutsos, M.~Ranganathan, and Y.~Manolopoulos, ``Fast subsequence matching
  in time-series databases,'' \emph{ACM Sigmod Record}, vol.~23, no.~2, pp.
  419--429, 1994.

\bibitem{chan1999efficient}
K.-P. Chan and A.~W.-C. Fu, ``Efficient time series matching by wavelets,'' in
  \emph{Proceedings 15th International Conference on Data Engineering (Cat. No.
  99CB36337)}.\hskip 1em plus 0.5em minus 0.4em\relax IEEE, 1999, pp. 126--133.

\bibitem{lin2003symbolic}
J.~Lin, E.~Keogh, S.~Lonardi, and B.~Chiu, ``A symbolic representation of time
  series, with implications for streaming algorithms,'' in \emph{Proceedings of
  the 8th ACM SIGMOD workshop on Research issues in data mining and knowledge
  discovery}, 2003, pp. 2--11.

\bibitem{schafer2012sfa}
P.~Sch{\"a}fer and M.~H{\"o}gqvist, ``{SFA: A} symbolic fourier approximation
  and index for similarity search in high dimensional datasets,'' in
  \emph{Proceedings of the 15th international conference on extending database
  technology}, 2012, pp. 516--527.

\bibitem{paparrizos2019grail}
J.~Paparrizos and M.~J. Franklin, ``{GRAIL: E}fficient time-series
  representation learning,'' \emph{Proceedings of the VLDB Endowment}, vol.~12,
  no.~11, pp. 1762--1777, 2019.

\bibitem{yuan2019wave2vec}
Y.~Yuan, G.~Xun, Q.~Suo, K.~Jia, and A.~Zhang, ``{Wave2vec: D}eep
  representation learning for clinical temporal data,'' \emph{Neurocomputing},
  vol. 324, pp. 31--42, 2019.

\bibitem{fortuin2018som}
V.~Fortuin, M.~H{\"u}ser, F.~Locatello, H.~Strathmann, and G.~R{\"a}tsch,
  ``{SOM-VAE: I}nterpretable discrete representation learning on time series,''
  in \emph{International Conference on Learning Representations}, 2018.

\bibitem{ma2019learning}
Q.~Ma, J.~Zheng, S.~Li, and G.~W. Cottrell, ``Learning representations for time
  series clustering,'' \emph{Advances in neural information processing
  systems}, vol.~32, 2019.

\bibitem{yao2017trajectory}
D.~Yao, C.~Zhang, Z.~Zhu, J.~Huang, and J.~Bi, ``Trajectory clustering via deep
  representation learning,'' in \emph{2017 international joint conference on
  neural networks (IJCNN)}.\hskip 1em plus 0.5em minus 0.4em\relax IEEE, 2017,
  pp. 3880--3887.

\bibitem{franceschi2019unsupervised}
J.-Y. Franceschi, A.~Dieuleveut, and M.~Jaggi, ``Unsupervised scalable
  representation learning for multivariate time series,'' \emph{Advances in
  neural information processing systems}, vol.~32, 2019.

\bibitem{zerveas2021transformer}
G.~Zerveas, S.~Jayaraman, D.~Patel, A.~Bhamidipaty, and C.~Eickhoff, ``A
  transformer-based framework for multivariate time series representation
  learning,'' in \emph{Proceedings of the 27th ACM SIGKDD Conference on
  Knowledge Discovery \& Data Mining}, 2021, pp. 2114--2124.

\bibitem{zhou2021informer}
H.~Zhou, S.~Zhang, J.~Peng, S.~Zhang, J.~Li, H.~Xiong, and W.~Zhang,
  ``Informer: Beyond efficient transformer for long sequence time-series
  forecasting,'' in \emph{Proceedings of AAAI}, 2021.

\bibitem{zhang2017joint}
C.~Zhang and P.~C. Woodland, ``Joint optimisation of tandem systems using
  {Gaussian} mixture density neural network discriminative sequence training,''
  in \emph{2017 IEEE International Conference on Acoustics, Speech and Signal
  Processing (ICASSP)}.\hskip 1em plus 0.5em minus 0.4em\relax IEEE, 2017, pp.
  5015--5019.

\bibitem{fink2007important}
E.~Fink and H.~S. Gandhi, ``Important extrema of time series,'' 2007.

\bibitem{fu2006time}
T.-c. Fu, H.-p. Chan, F.-l. Chung, and C.-m. Ng, ``Time series subsequence
  searching in specialized binary tree,'' in \emph{International Conference on
  Fuzzy Systems and Knowledge Discovery}.\hskip 1em plus 0.5em minus
  0.4em\relax Springer, 2006, pp. 568--577.

\bibitem{fuchs2010online}
E.~Fuchs, T.~Gruber, J.~Nitschke, and B.~Sick, ``Online segmentation of time
  series based on polynomial least-squares approximations,'' \emph{IEEE
  Transactions on Pattern Analysis and Machine Intelligence}, vol.~32, no.~12,
  pp. 2232--2245, 2010.

\bibitem{sivaraks2015robust}
H.~Sivaraks and C.~A. Ratanamahatana, ``Robust and accurate anomaly detection
  in {ECG} artifacts using time series motif discovery,'' \emph{Computational
  and mathematical methods in medicine}, vol. 2015, 2015.

\bibitem{thuy2021efficient}
H.~T.~T. Thuy, D.~T. Anh, and V.~T.~N. Chau, ``Efficient segmentation-based
  methods for anomaly detection in static and streaming time series under
  dynamic time warping,'' \emph{Journal of Intelligent Information Systems},
  vol.~56, no.~1, pp. 121--146, 2021.

\bibitem{pang2018intelligent}
J.~Pang, D.~Liu, Y.~Peng, and X.~Peng, ``Intelligent pattern analysis and
  anomaly detection of satellite telemetry series with improved time series
  representation,'' \emph{Journal of Intelligent \& Fuzzy Systems}, vol.~34,
  no.~6, pp. 3785--3798, 2018.

\bibitem{yang2013trasmil}
W.~Yang, Y.~Gao, and L.~Cao, ``{TRASMIL}: A local anomaly detection framework
  based on trajectory segmentation and multi-instance learning,''
  \emph{Computer Vision and Image Understanding}, vol. 117, no.~10, pp.
  1273--1286, 2013.

\bibitem{liu2020anomaly}
F.~Liu, X.~Zhou, J.~Cao, Z.~Wang, T.~Wang, H.~Wang, and Y.~Zhang, ``Anomaly
  detection in quasi-periodic time series based on automatic data segmentation
  and attentional {LSTM-CNN},'' \emph{IEEE Transactions on Knowledge and Data
  Engineering}, 2020.

\bibitem{1996:temporal:anomaly}
F.~S, H.~SA, S.~A, and L.~TA, ``Time series anomaly detection based on shapelet
  learning,'' in \emph{Proceedings of the 1996 IEEE Symposium on Research on
  Security and Privacy}, 1996, pp. 120--128.

\bibitem{hyndman2015large}
R.~J. Hyndman, E.~Wang, and N.~Laptev, ``Large-scale unusual time series
  detection,'' in \emph{Proceedings of the 2015 IEEE International Conference
  on Data Mining Workshop (ICDMW)}.\hskip 1em plus 0.5em minus 0.4em\relax
  IEEE, 2015, pp. 1616--1619.

\bibitem{chatfield2016analysis}
C.~Chatfield, \emph{The Analysis of Time Series: {A}n Introductionn}.\hskip 1em
  plus 0.5em minus 0.4em\relax Chapman and Hall/CRC, 2016.

\bibitem{Matsubara2014a}
Y.~Matsubara, Y.~Sakurai, and C.~Faloutsos, ``{AutoPlait}: Automatic mining of
  co-evolving time sequences,'' in \emph{Proceedings on the ACM SIGMOD
  Conference}, 2014, pp. 193--204.

\bibitem{GMM-book}
G.~McLachlan and D.~Peel, \emph{Finite Mixture Models}.\hskip 1em plus 0.5em
  minus 0.4em\relax New York NY: John Wiley \& Sons, 2000, vol. 3204.

\bibitem{mao2020multitask}
C.~Mao, A.~Gupta, V.~Nitin, B.~Ray, S.~Song, J.~Yang, and C.~Vondrick,
  ``Multitask learning strengthens adversarial robustness,'' in \emph{Computer
  Vision--ECCV 2020: 16th European Conference, Glasgow, UK, August 23--28,
  2020, Proceedings, Part II 16}.\hskip 1em plus 0.5em minus 0.4em\relax
  Springer, 2020, pp. 158--174.

\bibitem{SGD-convergence-book}
L.~Bottou, \emph{Online Algorithms and Stochastic Approximations}.\hskip 1em
  plus 0.5em minus 0.4em\relax New York NY: Cambridge University Press, 1998.

\bibitem{EM-convergence}
C.~Wu, ``On the convergence properties of the em algorithm,'' \emph{The Annals
  of Statistics}, vol.~11, no.~1, pp. 95--103, 1983.

\bibitem{graves2007multi}
A.~Graves, S.~Fern{\'a}ndez, and J.~Schmidhuber, ``Multi-dimensional recurrent
  neural networks,'' in \emph{International conference on artificial neural
  networks}.\hskip 1em plus 0.5em minus 0.4em\relax Springer, 2007, pp.
  549--558.

\bibitem{4781172}
V.~{Chandola}, V.~{Mithal}, and V.~{Kumar}, ``Comparative evaluation of anomaly
  detection techniques for sequence data,'' in \emph{Proceedings of the
  International Conference on Data Mining}, 2008, pp. 743--748.

\bibitem{yeung2002parzen}
D.-Y. Yeung and C.~Chow, ``Parzen-window network intrusion detectors,'' in
  \emph{Object Recognition Suppoted by User Interaction for Service Robots},
  vol.~4, 2002, pp. 385--388.

\bibitem{bevilacqua2015dictionary}
M.~Bevilacqua and S.~Tsaftaris, ``Dictionary-decomposition-based one-class
  {SVM} for unsupervised detection of anomalous time series,'' in
  \emph{Proceedings of 23rd European Signal Processing Conference (EUSIPCO)},
  2015, pp. 1776--1780.

\bibitem{scholkopf1999support}
B.~Sch{\"o}lkopf, R.~C. Williamson, A.~J. Smola, J.~Shawe-Taylor, J.~C. Platt
  \emph{et~al.}, ``Support vector method for novelty detection.'' in
  \emph{NIPS}, vol.~12.\hskip 1em plus 0.5em minus 0.4em\relax Citeseer, 1999,
  pp. 582--588.

\bibitem{yeh2016matrix}
C.-C.~M. Yeh, Y.~Zhu, L.~Ulanova, N.~Begum, Y.~Ding, H.~A. Dau, D.~F. Silva,
  A.~Mueen, and E.~Keogh, ``{Matrix profile I: A}ll pairs similarity joins for
  time series: A unifying view that includes motifs, discords and shapelets,''
  in \emph{2016 IEEE 16th international conference on data mining
  (ICDM)}.\hskip 1em plus 0.5em minus 0.4em\relax Ieee, 2016, pp. 1317--1322.

\bibitem{malhotra2016lstm}
P.~Malhotra, A.~Ramakrishnan, G.~Anand, L.~Vig, P.~Agarwal, and G.~Shroff,
  ``{LSTM}-based encoder-decoder for multi-sensor anomaly detection,''
  \emph{ICML Anomaly Detection Workshop}, 2016.

\bibitem{ryan2019pattern}
S.~Ryan, R.~Corizzo, I.~Kiringa, and N.~Japkowicz, ``Pattern and anomaly
  localization in complex and dynamic data,'' in \emph{2019 18th IEEE
  International Conference On Machine Learning And Applications (ICMLA)}.\hskip
  1em plus 0.5em minus 0.4em\relax IEEE, 2019, pp. 1756--1763.

\bibitem{dong2020rtmobile}
P.~Dong, S.~Wang, W.~Niu, C.~Zhang, S.~Lin, Z.~Li, Y.~Gong, B.~Ren, X.~Lin, and
  D.~Tao, ``Rtmobile: Beyond real-time mobile acceleration of {RNNs} for speech
  recognition,'' in \emph{2020 57th ACM/IEEE Design Automation Conference
  (DAC)}.\hskip 1em plus 0.5em minus 0.4em\relax IEEE, 2020, pp. 1--6.

\bibitem{wang2018c}
S.~Wang, Z.~Li, C.~Ding, B.~Yuan, Q.~Qiu, Y.~Wang, and Y.~Liang, ``{C-LSTM}:
  Enabling efficient {LSTM} using structured compression techniques on
  {FPGAs},'' in \emph{Proceedings of the 2018 ACM/SIGDA International Symposium
  on Field-Programmable Gate Arrays}, 2018, pp. 11--20.

\end{thebibliography}


\begin{IEEEbiography}[{\includegraphics[width=1in,height=1.25in,clip,keepaspectratio]{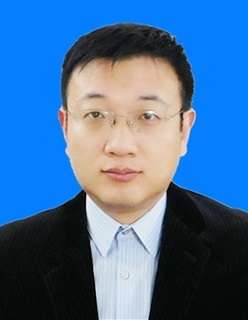}}]
 {Kai Yang} (SM'18) received the B.Eng. degree from Southeast University, Nanjing, China, the M.S. degree from the National University of Singapore, Singapore, and the Ph.D. degree from Columbia
University, New York, NY, USA.

He is a Distinguished Professor with Tongji University, Shanghai, China. He was a Technical
Staff Member with Bell Laboratories, Murray Hill, NJ, USA. He has also been an Adjunct Faculty
Member with Columbia University since 2011. He holds over 20 patents and has been published extensively in leading IEEE journals and conferences. His current research interests include big data analytics, machine learning, wireless communications, and signal processing.
\end{IEEEbiography}

\begin{IEEEbiography}[{\includegraphics[width=1in,height=1.25in,clip,keepaspectratio]{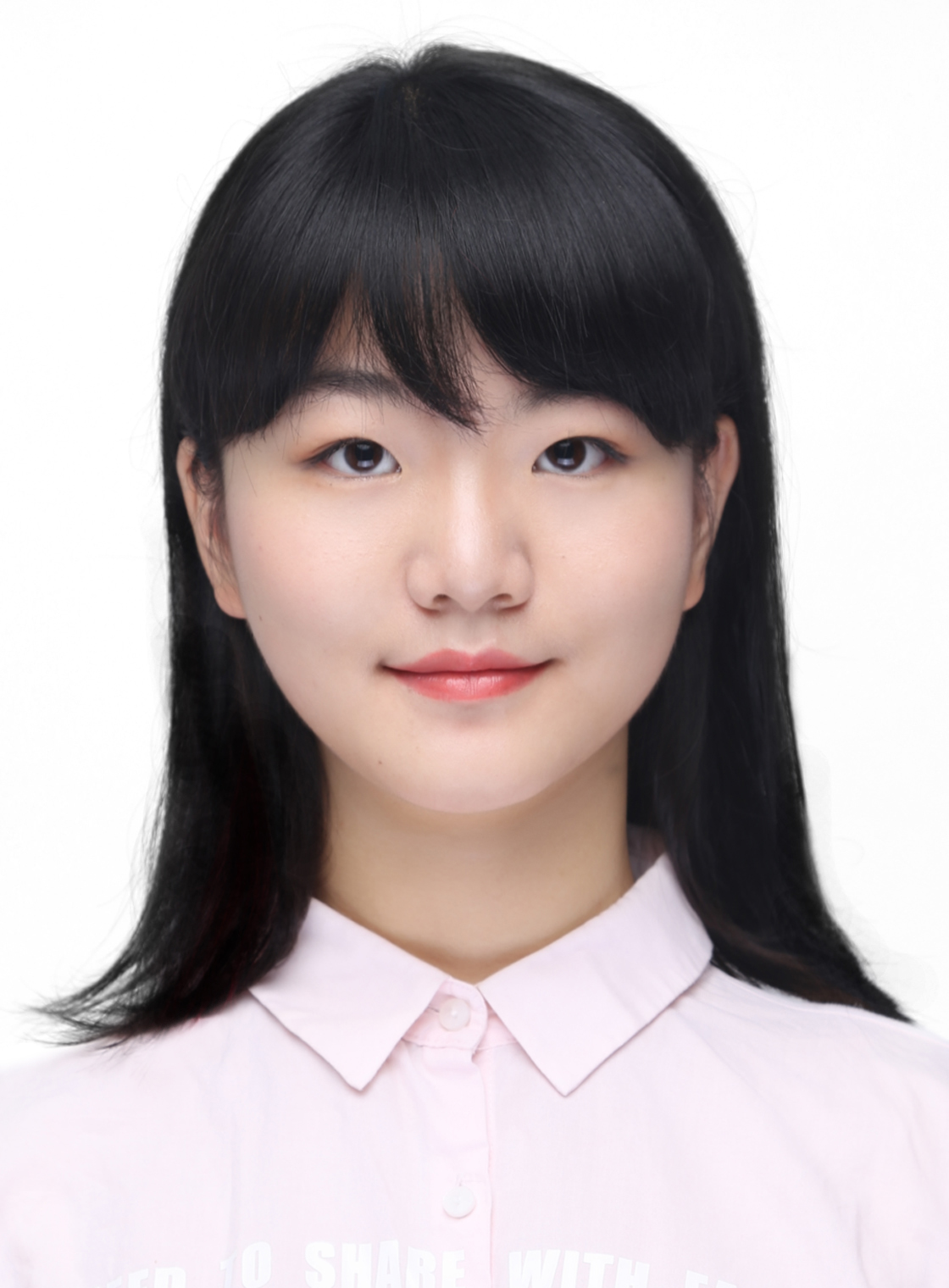}}]
{Shaoyu Dou} was born in Gansu, China, in 1996. She received
a B.Eng. degree from Hohai University, Nanjing, China, in 2018. She is currently pursuing a Ph.D. degree in computer science in the Department of Computer Science at Tongji University, Shanghai, China. Her major research interests include big data analytics and machine learning.
\end{IEEEbiography}

\begin{IEEEbiography}[{\includegraphics[width=1in,height=1.25in,clip,keepaspectratio]{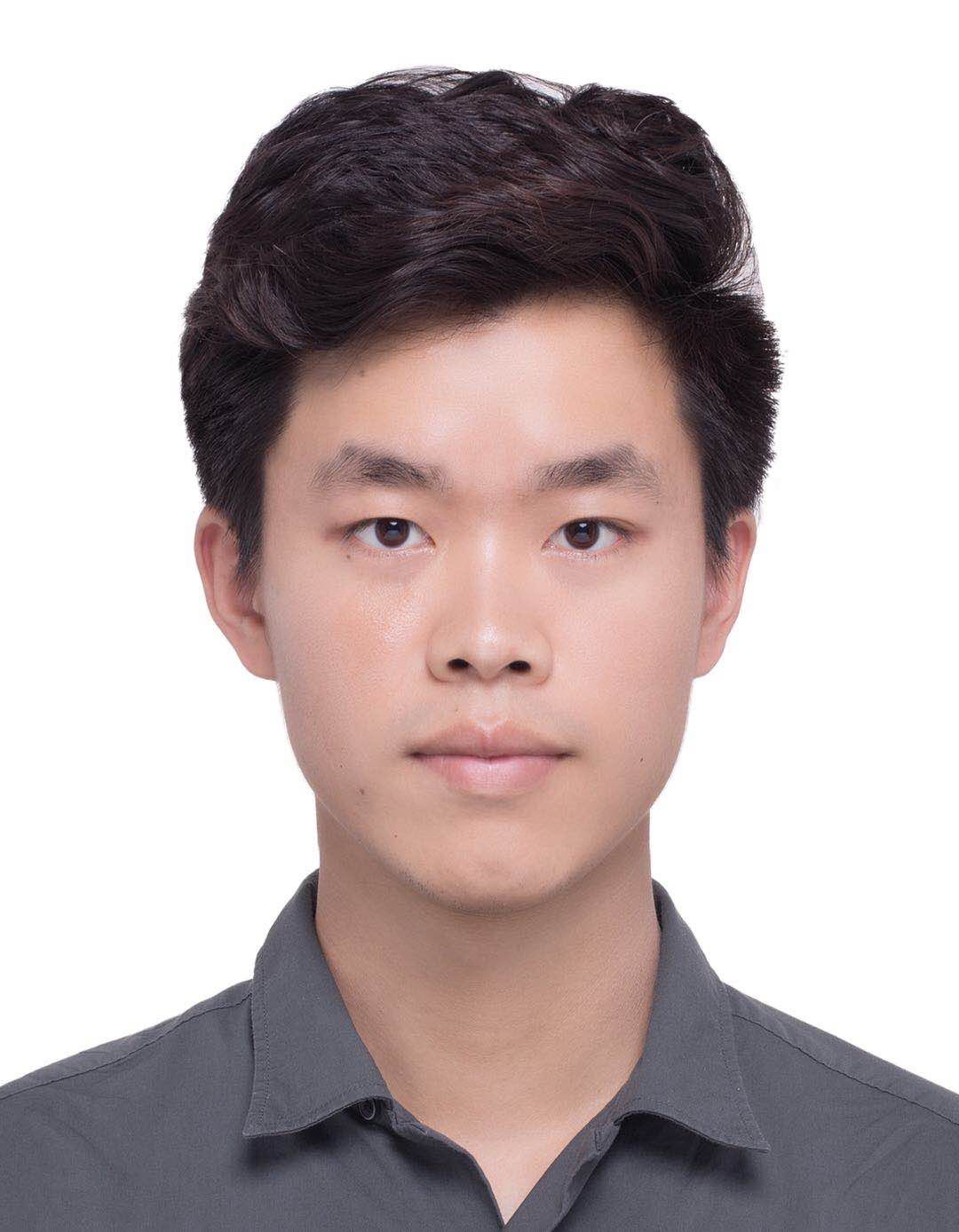}}]
{Pan Luo} was born in Sichuan, China, in 1995.
He received the B.Eng. degree from Tongji University,
Shanghai, China, in 2018. He continued to pursue his Master degree in computer science with the Department of Computer Science, Tongji University and successfully graduated in 2021. His  research focuses on data mining and machine learning.
\end{IEEEbiography}

\begin{IEEEbiography}[{\includegraphics[width=1in,height=1.25in,clip,keepaspectratio]{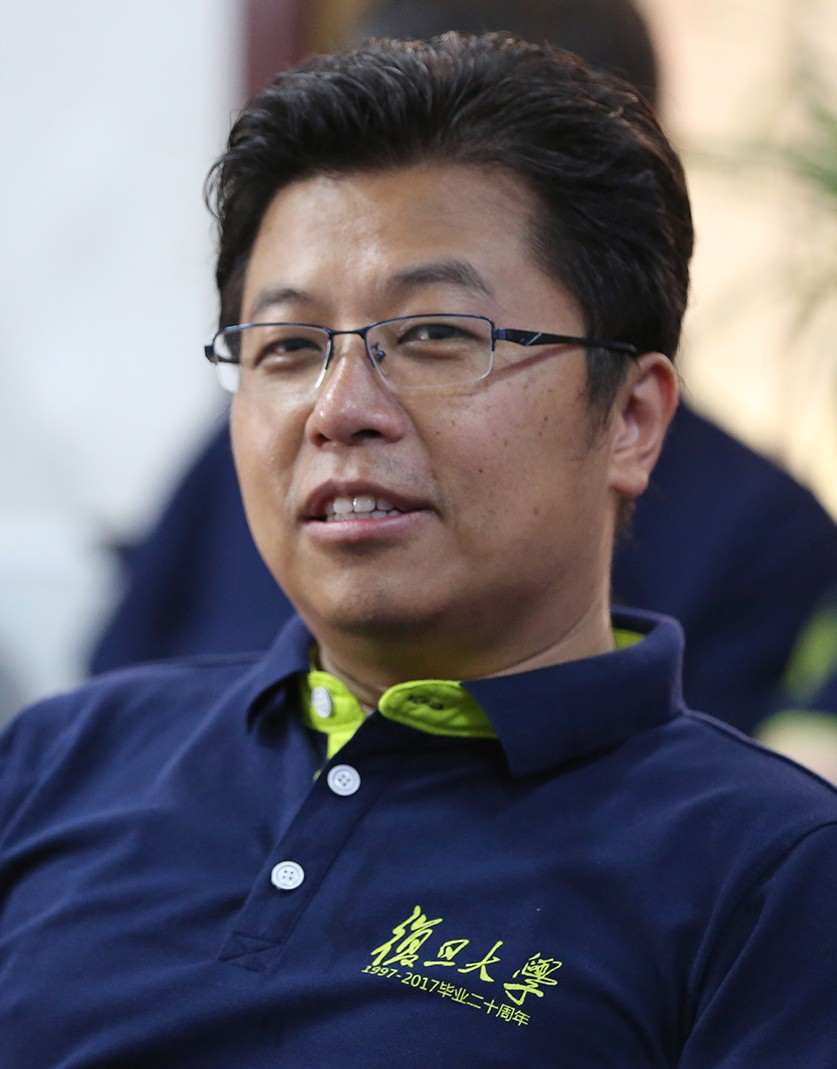}}]
{Xin Wang} (SM'09) received the B.Sc. and M.Sc.degrees from Fudan University, Shanghai, China, in 1997 and 2000, respectively, and the Ph.D. degree from Auburn University, Auburn, AL, USA, in 2004, all in electrical engineering.

From September 2004 to August 2006, he was a Postdoctoral Research Associate with the Department of Electrical and Computer Engineering, University of Minnesota, Minneapolis. In August 2006, he joined the Department of Electrical Engineering, Florida Atlantic University, Boca Raton, FL, USA, as an Assistant Professor, then was promoted to a tenured Associate Professor in 2010. He is currently a Distinguished Professor and the Chair of the Department of Communication Science and Engineering, Fudan University, China. His research interests include stochastic network optimization, energy-efficient communications, cross-layer design, and signal processing for communications. He is a Senior Area Editor for the IEEE Transactions on Signal Processing and an Editor for the IEEE Transactions on Wireless Communications, and in the past served as an Associate Editor for the IEEE Transactions on Signal Processing, as an Editor for the IEEE Transactions on Vehicular Technology, and as an Associate Editor for the IEEE Signal Processing Letters. He is a member of the Signal Processing for Communications and Networking Technical Committee of IEEE Signal Processing Society, and a Distinguished Lecturer of the IEEE Vehicular Technology Society.
\end{IEEEbiography}

\begin{IEEEbiography}[{\includegraphics[width=1in,height=1.25in,clip,keepaspectratio]{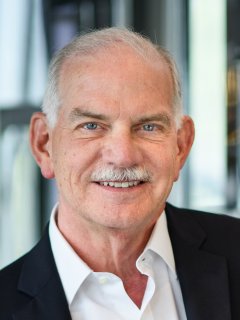}}]
{H. Vincent Poor} (S'72, M'77, SM'82, F'87) received the Ph.D. degree in EECS from Princeton University in 1977.  From 1977 until 1990, he was on the faculty of the University of Illinois at Urbana-Champaign. Since 1990 he has been on the faculty at Princeton, where he is currently the Michael Henry Strater University Professor. During 2006 to 2016, he served as the dean of Princeton’s School of Engineering and Applied Science. He has also held visiting appointments at several other universities, including most recently at Berkeley and Cambridge. His research interests are in the areas of information theory, machine learning and network science, and their applications in wireless networks, energy systems and related fields. Among his publications in these areas is the forthcoming book Machine Learning and Wireless Communications.  (Cambridge University Press). Dr. Poor is a member of the National Academy of Engineering and the National Academy of Sciences and is a foreign member of the Chinese Academy of Sciences, the Royal Society, and other national and international academies. He received the IEEE Alexander Graham Bell Medal in 2017.
\end{IEEEbiography}







\end{spacing}
\end{document}